\documentclass[journal]{IEEEtran}
\usepackage{graphicx}
\usepackage{picinpar}
\usepackage{amsmath,mathrsfs}
\usepackage{stfloats}
\usepackage{url}
\usepackage{flushend}
\usepackage{colortbl}
\usepackage{xcolor,soul,framed} 
\colorlet{shadecolor}{yellow}
\usepackage{multirow}
\usepackage{pifont}
\usepackage{color,soul}
\usepackage{alltt}
\usepackage[hidelinks]{hyperref}
\usepackage{enumerate}
\usepackage{siunitx}
\usepackage{breakurl}
\usepackage{epstopdf}
\usepackage{pbox}
\usepackage{amssymb}
\usepackage{amsthm}
\usepackage{comment}
\usepackage{amsthm}

\newtheorem{lemm}{Lemma}
\newtheorem{prop}{Proposition}
\newtheorem{definition}{Definition}
\usepackage{subfigure}
\usepackage{algorithm}
\usepackage{algpseudocode}

\renewcommand{\algorithmicrequire}{\textbf{Input:}}  
\renewcommand{\algorithmicensure}{\textbf{Output:}}
\usepackage{autobreak}
\usepackage[percent]{overpic} 
\usepackage{rotating}
\graphicspath{{../figures/}}
\DeclareGraphicsExtensions{.eps,.eps_tex,.pdf,.jpeg,.png}
\definecolor{limegreen}{rgb}{0.2, 0.8, 0.2}
\definecolor{forestgreen}{rgb}{0.13, 0.55, 0.13}
\definecolor{greenhtml}{rgb}{0.0, 0.5, 0.0}
\definecolor{skyblue}{rgb}{0.53, 0.81, 0.92}
\definecolor{lightgray}{rgb}{0.83, 0.83, 0.83}
\definecolor{gray}{rgb}{0.75, 0.75, 0.75}
\definecolor{darkgray}{rgb}{0.66, 0.66, 0.66}
\usepackage{mathtools}
\colorlet{shadecolor}{yellow!255}
\usepackage{tikz}
\newcommand\copyrighttext{%
  \footnotesize This work has been submitted to the IEEE for possible publication. Copyright may be transferred without notice, after which this version may no longer be accessible.}
\newcommand\copyrightnotice{%
\begin{tikzpicture}[remember picture,overlay]
\node[anchor=south,yshift=8pt] at (current page.south) {\fbox{\parbox{\dimexpr\textwidth-\fboxsep-\fboxrule\relax}{\copyrighttext}}};
\end{tikzpicture}%
}

\begin{document}
\title{Generalized Flexible Hybrid Cable-Driven Robot (HCDR): Modeling, Control, and Analysis}

\author{
	\vskip 1em
    {Ronghuai Qi, Amir Khajepour, and William W. Melek}
\thanks{
{
This work was supported in part by the Natural Sciences and Engineering Research Council of Canada.

R. Qi, A. Khajepour, and W. W. Melek are with the Department of Mechanical and Mechatronics Engineering, University of Waterloo, Waterloo, ON N2L 3G1, Canada (e-mail: ronghuai.qi@uwaterloo.ca; a.khajepour@uwaterloo.ca; william.melek@uwaterloo.ca).
		}
	}
}

\maketitle
{\color{black}
\begin{abstract}
This paper presents a generalized flexible Hybrid Cable-Driven Robot (HCDR). For the proposed HCDR, the derivation of the equations of motion and proof provide a very effective way to find items for generalized system modeling. The proposed dynamic modeling approach avoids the drawback of traditional methods and can be easily extended to other types of hybrid robots, such as a robot arm mounted on an aircraft platform.

Additionally, another goal of this paper is to develop integrated control systems to reduce vibrations and improve the accuracy and performance of the HCDR. To achieve this goal, redundancy resolution, stiffness optimization, and control strategies are studied. The proposed optimization problem and algorithm address the limitations of existing stiffness optimization approaches. Three types of control architecture are proposed, and their performances (i.e., reducing undesirable vibrations and trajectory tracking errors, especially for the end-effector) are evaluated using several well-designed case studies. Results show that the fully integrated control strategy can improve the tracking performance of the end-effector significantly.
\end{abstract}

}

\copyrightnotice

\section{Introduction}
\IEEEPARstart{S}{erial} manipulators are one of the most common types of industrial robots, which consist of a base, a series of links connected by motor-driven joints, and an end-effector. Usually, they have 6 degrees of freedom (DOFs) and offer high positioning accuracy. They are commonly used in industrial applications; however, they have some key limitations, such as high motion inertia and limited workspace envelope \cite{Wei2015}. Cable-driven parallel robots (CDPRs) are another important type of industrial robots. Their configurations usually bear resemblance to parallel manipulators. For these robots, rigid links are replaced with cables. This reduces the robot weight since cables are almost massless. It also eliminates the use of revolute joints. These features allow the mobile platform to reach high motion accelerations in large workspaces. However, they are not without some drawbacks, such as their low accuracy, high vibration, etc., all of which limit their applications \cite{Oh2005}. To overcome the aforementioned shortages of serial and cable-driven parallel robots as well as combine their advantages, one approach is to combine these two types of robots to create a hybrid cable-driven robot (HCDR), i.e., a hybrid structure of CDPR(s) and serial robot(s).

The literature shows that existing research and applications prefer to affix a robot arm upside down to the bottom of a CDPR\textrm{'}s platform \cite{T.Arai1999,H.Osumi2000,M.Bamdad2015,M.Gouttefarde2017,J.S.Albus1989,J.S.Albus2003,GmbH,SKYCAM_LLC} or mainly control the cable robot while treating the serial robot as a manipulation tool or an end-effector rather than a whole system \cite{J.S.Albus1989,J.S.Albus2003,GmbH,SKYCAM_LLC}. When a serial robot is mounted on a mobile platform, they constitute a new coupled system. Only controlling the mobile platform (i.e., treating the serial robot as a manipulation tool) or the serial robot may not guarantee the position accuracy of the end-effector. For applications that use such a system, the main goal is to control the end-effector of the serial robot (e.g., its trajectories and vibrations) to effectively accomplish tasks such as pick-and-place. Another major challenge in the utilization of these systems is maintaining the appropriate cable tensions and stiffness for the robot. This requires the development of kinematic and dynamic models, stiffness optimization, and controllers for HCDRs.

Some research has been carried out to solve these problems: for kinematic and dynamic modeling, existing research mainly focuses on rigid serial robots \cite{P.R.Pagilla2004}, rigid/flexible parallel robots \cite{Lau2013,N.Mostashiri2018, C.Viegas2017, Khajepour2015, Taghirad2011,H.Jamshidifar2017,Z.Mu2018,Otis2009,M.Chen2018}, and wheeled rigid mobile vehicles carrying a rigid/flexible joint arm \cite{Lin2001,S.A.Marinovic2017}. In \cite{Lau2013}, a multilink manipulator model was developed, but this model applied to each link driven by cables. To solve the redundancy and stiffness optimization problems, some useful methods were studied, such as minimum 2-norm of cable tensions{~}\cite{Khajepour2015,Mendez2014} and stiffness maximization in the softest direction \cite{H.Jamshidifar2017, H.Jamshidifar2018}. However, their research focused on planar CDPRs, and maximizing robots\textrm{'} overall stiffness by using these approaches was not always guaranteed. Since the use of flexible cables reduces the overall stiffness of cable-driven robots, vibration control becomes a serious problem. {\color{black}Meanwhile, the moving robot arm also generates reaction forces to the mobile platform, resulting in the mobile platform vibrating. Hence, it is challenging to achieve the goal of minimizing the vibrations and increasing the position accuracy of the end-effector.} To the best of my knowledge, limited studies address the modeling and control problems of flexible HCDRs. Especially, when the redundancy and stiffness optimization problems are introduced, the control of trajectories and vibrations becomes more challenging. {\color{black}Researchers in{~}\cite{Tecnalia} showed a CDPR carrying a robot arm for painting large surfaces, but vibrations were obvious and large based on their demonstration.}

This paper is motivated by the need to solve the aforementioned problems for CDPRs with serial robotic arms in order to increase their accuracy and adoption in industrial or other potential applications (e.g., rehabilitation). To implement this motivation, this paper focuses on a generalized flexible HCDR (shown in \autoref{fig:J1_GeneralizedHCDR}), including modeling, control, and performance analysis. {\color{black}The novelty and main contributions of this paper are as follows:}

{\color{black}
\begin{enumerate}
\item {The derivation of the equations of motion and proof provide a very effective way to find items for generalized system modeling. Meanwhile, the proposed dynamic modeling approach avoids the drawback of traditional methods (e.g., \cite{Denavit1955}), and can be easily extended to other types of hybrid robots by changing the proposed structure matrix based on their desired configurations, e.g., robot arm(s) mounted on an aircraft platform{~}{\cite{Tardella2016,PRODRONE2016}}.}

\item {Three types of control architecture are proposed to reduce vibrations and improve the accuracy of the HCDR. Their performances are also evaluated using several well-designed case studies.}

\item {The proposed optimization problem and algorithm address the limitations of existing stiffness optimization approaches in{~}{\cite{Khajepour2015,Mendez2014,H.Jamshidifar2017, H.Jamshidifar2018}}. Meanwhile, they can be applied to not only CDPRs but also HCDRs.}
\end{enumerate}}

Additionally, the growth of automated warehousing solutions has been fueled by the e-commerce explosion in recent years~\cite{Hamedthesis2018}. By 2024, the market of global automated material handling equipment is predicted to no less than US\$ 50.0 Billion with a Compound Annual Growth Rate (CAGR) of $8\%$~\cite{Heraldkeeper2019,MarketEngine2018}. These increase of automated warehousing applications offers a unique opportunity for the development of cable robotics that is superior in, especially for~{CDPRs and HCDRs}. The proposed technique in this paper provides a valid solution for the development of~{CDPRs and HCDRs}.

In this paper, generalized system modeling is introduced in \autoref{sec:J1_GeneralizedSystemModeling}. In \autoref{sec:J1_HCDPRExampleforAnalysis}, a HCDR example is selected by applying the modeling method in \autoref{sec:J1_GeneralizedSystemModeling}. Then, In \autoref{sec:J1_VibCtrlDesign}, vibration control design based on this HCDR example is implemented. Control performance and evaluation using case studies are presented in \autoref{sec:J1_ControlPerformanceandEva}. Finally, in \autoref{sec:J1_Conclusions}, the contributions of this paper are summarized.

\section{Generalized System Modeling}\label{sec:J1_GeneralizedSystemModeling}
In this paper, a generic hybrid cable-driven robot (HCDR) is proposed to overcome the shortcomings of CDPRs and serial robots as well as aggregate their advantages. {\color{black}The HCDR is defined as follows:
\begin{definition}\label{definition:hybrid}
A hybrid cable-driven robot (HCDR) is a robot that is composed of two or more heterogeneous mechatronic components, where at least one component is CDPR.
\end{definition}

With reference to Definition~\autoref{definition:hybrid}, let us consider} a generalized $(n+m)$-DOF HCDR in three-dimensional (3D) space (shown in \autoref{fig:J1_GeneralizedHCDR}) with an $n$-DOF ($\{n \in {\mathbb{N}}:n \le 6\}$) cable-driven parallel robot (mobile platform) and an $m$-DOF ($m \in {\mathbb{N}}$) robot arm, where the robot arm is mounted on the mobile platform and moves with it. To simplify modeling, all the driven cables are assumed massless, straight, and stretchable.

As a coupled system, modeling is much harder by comparison to just parallel robot or serial robot arm, especially when flexible parts are introduced (e.g., flexible driven cables). To develop the model of the hybrid system, first, we derive the equations of motion of the $n$-DOF CDPR (in \autoref{subsec:J1_CDPRDyn}); then, we will use some results in \autoref{subsec:J1_CDPRDyn} to derive the equations of motion of the $(n+m)$-DOF HCDR (in \autoref{subsec:J1_HCDPRDyn}).
\begin{figure}[h]\centering
	\includegraphics[width=8.6cm]{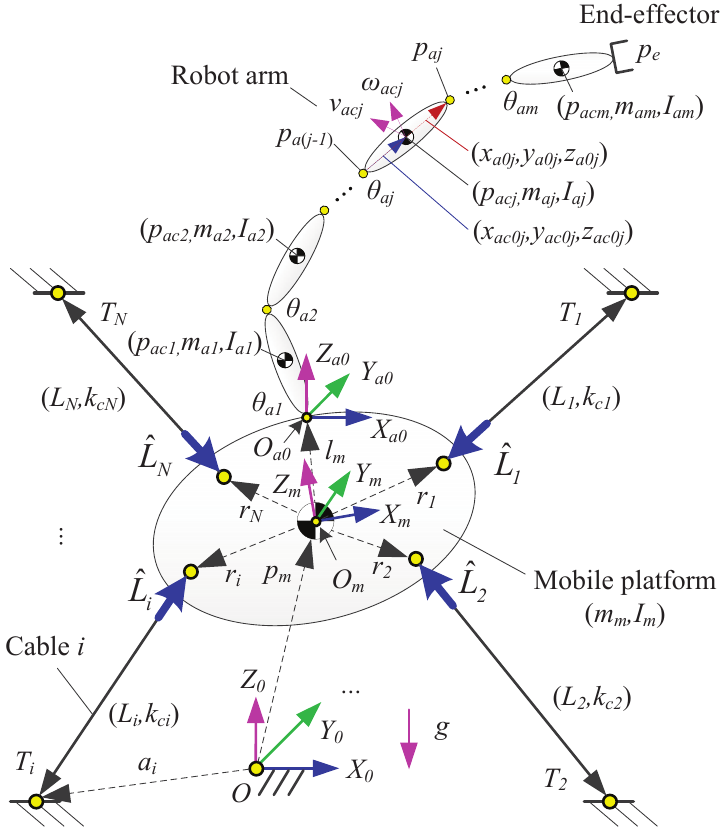}
	\vspace{-0.2cm}
	\caption{Configuration of a generalized $(n+m)$-DOF HCDR with an $n$-DOF CDPR and an $m$-DOF robot arm, where the robot arm is mounted on the CDPR.}\label{fig:J1_GeneralizedHCDR}
\vspace{-0.2cm}
\end{figure}

\subsection{Equations of Motion of the CDPR}\label{subsec:J1_CDPRDyn}
In \autoref{fig:J1_GeneralizedHCDR}, the inertial coordinate frame $\{O\}$ is assumed fixed on the base/ground. Coordinate frame $\{O_m\}$ is located at the center of mass (COM) of the mobile platform. By assuming the Euler angles ${[{\alpha _m},{\beta _m},{\gamma _m}]^T} \in {\mathbb{R}^3}$ (the orientations of the mobile platform about $X$-, $Y$-, and $Z$-axes, respectively), the rotation matrix (e.g., $X\rightarrow Y'\rightarrow Z''$ order as below) is computed as
\begin{align}
R_g^m = R_x({\alpha _m})R_{y'}({\beta _m})R_{z''}({\gamma _m}) \in {\mathcal{SO}(3)}. \label{eq:J1_1}
\end{align}

Then, the cable-length vector is calculated as
\begin{align}
{{{\vec L}}_i} = &{[{p_{mx}},{p_{my}},{p_{mz}}]^T} + R_g^m{[r_{ix},r_{iy},r_{iz}]^T} \nonumber \\ 
&{- {[a_{ix},a_{iy},a_{iz}]^T}}
,\; \{\forall \; i \in{\mathbb{N}} :1 \le i \le N\}
\label{eq:J1_2}
\end{align}
where ${{{\vec L}}_i} \in {\mathbb{R}^3}$ denotes the position vector from the $i$th cable anchor point on the robot static frame to the $i$th cable anchor point on the mobile platform; ${{p_m}:=[{p_{mx}},{p_{my}},{p_{mz}}]^T} \in {\mathbb{R}^3}$ represents the position vector of the coordinate frame $\{O_m\}$ with respect to the coordinate frame $\{O\}$; ${{r_i}:=[r_{ix},r_{iy},r_{iz}]^T} \in {\mathbb{R}^3}$ denotes the position vector of the $i$th cable anchor point on the mobile platform with respect to the body-fixed frame $\{O_m\}$; ${{a_i}:=[a_{ix},a_{iy},a_{iz}]^T} \in {\mathbb{R}^3}$ represents the position vector of the $i$th cable anchor point on the robot static frame with respect to the coordinate frame $\{O\}$; and $N$ is the total number of cables. Then, the $i$th cable length ${L_i} \in \mathbb{R}$ is computed as
\begin{align}
L_i = &{\lVert {{[{p_{mx}},{p_{my}},{p_{mz}}]}^T} + R_g^m{{[{r_{ix}},{r_{iy}},{r_{iz}}]}^T}}\nonumber \\
&{- {{[{a_{ix}},{a_{iy}},{a_{iz}}]}^T}} \rVert.
\label{eq:J1_3}
\end{align}

In addition, the derivative of \eqref{eq:J1_2} is rearranged as
\begin{align}
{{\dot L}_i} = {{\hat L}}_i^T{{{v}}_m} + {(R_g^m{[r_{ix},r_{iy},r_{iz}]^T} \times {{{{\hat L}}}_i})^T}{{\color{black}R_g^m}{{\omega }}_m}
\label{eq:J1_4}
\end{align}
where ${\dot L_i} \in \mathbb{R}^3$ denotes the $i$th cable length velocity, ${{{\hat L}}_i:=\frac{{\vec L}_i}{L_i}= {[{\hat L_{ix}},{\hat L_{iy}},{\hat L_{iz}}]^T}} \in {\mathbb{R}^3}$ represents the unit cable position vector, and ${{{v}}_m},{{{\omega }}_m} \in {\mathbb{R}^3}$ are the linear velocity and angular velocity of the coordinate frame $\{O_m\}$, respectively.  Then, \eqref{eq:J1_4} can be expanded in matrix form as
\begin{align}
&{[{{{\dot L}_1}},{{{\dot L}_2}}, \cdots ,{{{\dot L}_N}}]^T} = \nonumber\\
&{\underbrace{\begin{bmatrix}
{{\hat L}_1} & \cdots & {{\hat L}_N} \\
{R_g^m\begin{bmatrix} r_{1x} \\ r_{1y} \\ r_{1z} \\ \end{bmatrix}} \times {{\hat L}_1} & \cdots & {R_g^m\begin{bmatrix} r_{Nx} \\ r_{Ny} \\ r_{Nz} \\ \end{bmatrix}} \times {{\hat L}_N}\end{bmatrix}}_{=:A_m}}^T
\begin{bmatrix}
v_m \\
{\color{black}R_g^m}{\omega _m}
\end{bmatrix}
\label{eq:J1_5}
\end{align}
where $A_m$ represents a structure matrix, determined by the position and orientation of the mobile platform. The linear velocity ${{{v}}_m}$ and angular velocity ${{{\omega }}_m}$ are calculated as
\begin{subequations}
\begin{align}
\left[ \begin{array}{cc}
v_m \\
\hline
{\color{black}R_g^m}{\omega _m}
\end{array} \right] = &{({A_m^T})^ + }{[{{{\dot L}_1}},{{{\dot L}_2}}, \cdots ,{{{\dot L}_N}}]^T} \label{eq:J1_6a}\\
= &\left[ \begin{array}{ll}
{[{\dot p_{mx}},{\dot p_{my}},{\dot p_{mz}}]^T}\\
\hline
{\color{black}R_g^m}({[R_g^m]^T}{[{{\dot \alpha }_m},0,0]^T} + {[R_{m\beta }^{m\gamma }]^T}{[R_{m\alpha }^{m\beta }]^T}\\
{[0,{{\dot \beta }_m},0]^T} + {[R_{m\beta }^{m\gamma }]^T}{[0,0,{{\dot \gamma }_m}]^T})
\end{array} \right] \label{eq:J1_6b}
\end{align}
\end{subequations}
where ${(\cdot)^+}$ is the pseudo-inverse of the matrix $(\cdot)$, rotation matrices $R_{m\alpha}^{m\beta}=R_{y'}({\beta _m}) \in {\mathcal{SO}(3)}$ and $R_{m\beta}^{m\gamma}=R_{z''}({\gamma _m}) \in {\mathcal{SO}(3)}$ come from 
\eqref{eq:J1_1}. \eqref{eq:J1_6a} and \eqref{eq:J1_6b} are two expressions to compute ${{{v}}_m}$ and ${{{\omega }}_m}$. ${\dot p_{mx}},{\dot p_{my}},{\dot p_{mz}},{{\dot \alpha }_m},{{\dot \beta }_m}$, and ${{\dot \gamma }_m}$ are the time-derivative of ${p_{mx}},{p_{my}},{p_{mz}},{{\alpha }_m},{{\beta }_m}$, and ${{\gamma }_m}$, respectively.

For the CDPR dynamics, the Newton-Euler equations are used because they can describe the system in \autoref{fig:J1_GeneralizedHCDR} in terms of cable tensions directly. Then, we get
\begin{align}
&\begin{bmatrix}
{{m_m}{{{{\dot v}}}_m}}\\
{{\color{black}R_g^m}{{{I}}_m}{{{{\dot \omega }}}_m} + {\color{black}R_g^m}{{{\omega }}_m} \times ({{{I}}_m}{{{\omega }}_m})}
\end{bmatrix} + \begin{bmatrix}
{{{m_m}{{[0,0,g]}^T}}+{{{{F}}_e}}}\\
{{{{M}}_e}}
\end{bmatrix} \nonumber\\
&= \begin{bmatrix}
{\sum\limits_{i = 1}^N {({T_i}{{{{\hat L}}}_i})} }\\
{\sum\limits_{i = 1}^N {[{T_i}(R_g^m{[r_{ix},r_{iy},r_{iz}]^T} \times {{{{\hat L}}}_i})]} }
\end{bmatrix} = A_mT
\label{eq:J1_7}
\end{align}
where ${T_i} \in \mathbb{R}$ denotes the $i$th cable tension; ${{{\hat L}}_i} \in {\mathbb{R}^3}$ represents the unit vector of $i$th cable position; ${{{F}}_e}, {{{M}}_e} \in {\mathbb{R}^3}$ are the external forces and moments  (e.g., the interaction forces and torques from the mounted robot arm affecting the mobile platform) applied to the coordinate frame $\{O_m\}$; ${m_m} \in \mathbb{R}$ is the mass of the mobile platform; ${{{I}}_m} \in {\mathbb{R}^{3 \times 3}}$ denotes the moment of inertia of the mobile platform; ${{{v}}_m},{{{\dot v}}_m},{{{\omega }}_m},{{{\dot \omega }}_m} \in {\mathbb{R}^3}$ represent the linear velocity, linear acceleration, angular velocity, and angular acceleration the mobile platform, respectively; and $g$ is the gravitational acceleration.

Suppose the cable stiffness matrix is ${{{K}}_c} = {\color{black}{\rm{diag}}\left[{{k_{c1}},{k_{c2}}, \cdots ,{k_{cN}}} \right]} \in {\mathbb{R}^{N \times N}}$, where ${k_{ci}}=\frac{EA_i}{L_{0i}}$ represents the $i$th cable stiffness, 
$EA_i$ is the product of the modulus of elasticity and cross-sectional area  of the $i$th cable, and $L_{0i}$ denotes $i$th unstretched cable length. Then, the cable tension vector is calculated as
\begin{align}
T = {K_c}\left(L - L_0 \right)
\label{eq:J1_8}
\end{align}
where ${{T}} \in {\mathbb{R}^N}$ denotes the cable tension vector, ${{L}} \in {\mathbb{R}^N}$ represents the cable length vector, and ${{{L}}_0} \in {\mathbb{R}^N}$ denotes the vector of unstretched cable lengths. The directions of $T$ are shown in \autoref{fig:J1_GeneralizedHCDR}.

Finally, by considering a vector of unknown bounded disturbances ${\tau _{md}}$, \eqref{eq:J1_7} and \eqref{eq:J1_8} can be described as

\begin{align}
&\underbrace {\begin{bmatrix}
{{m_m}{I}}&\textbf{0}\\
\textbf{0}&{{\color{black}R_g^m}{{{I}}_m}}
\end{bmatrix}}_{=:{M_m}({q_m})}\underbrace {\begin{bmatrix}
{{{{{\dot v}}}_m}}\\
{{{{{\dot \omega }}}_m}}
\end{bmatrix}}_{=:{{\ddot q}_m}} + \underbrace {\begin{bmatrix}
\textbf{0}\\
{{\color{black}R_g^m}[{{{\omega }}_m}]{{{I}}_m}}
\end{bmatrix}}_{=:{C_m}({q_m},{{\dot q}_m})}\underbrace {\begin{bmatrix}
{{{{v}}_m}}\\
{{{{\omega }}_m}}
\end{bmatrix}}_{{=:{\dot q}_m}} \nonumber\\
&+ \begin{bmatrix}
{\underbrace {{m_m}{{[0,0,g]}^T}}_{=:{G_m}({q_m})} + {{{{F}}_e}}}\\
{{{{M}}_e}}
\end{bmatrix} + {\tau _{md}} = A_mT = :{{{\tau }}_m}
\label{eq:J1_9}
\end{align}
where \begin{small}
$[{{{\omega }}_m}]: = \begin{bmatrix}
0&{ - {{{\omega }}_{mz}}}&{{{{\omega }}_{my}}}\\
{{{{\omega }}_{mz}}}&0&{ - {{{\omega }}_{mx}}}\\
{ - {{{\omega }}_{my}}}&{{{{\omega }}_{mx}}}&0
\end{bmatrix}$
\end{small} and $I\in {\mathbb{R}^{3 \times 3}}$ is the identity matrix. ${M_m}({q_m})$, ${C_m}({q_m},{\dot q_m})$, and $G_m({q_m})$  denote the inertia matrix, Coriolis and centripetal matrix, and gravitational vector, respectively. $q_m$, ${\dot q}_m$, ${\ddot q}_m$, and ${\tau }_m$ represent the vectors of generalized coordinates, velocities, accelerations, and joint forces/torques, respectively. {\color{black}The derivation of equations of motion of the CDPR in this section provides a convenient closed form to simplify HCDR modeling in \autoref{subsec:J1_HCDPRDyn}.}

\subsection{Equations of Motion of the HCDR}
\label{subsec:J1_HCDPRDyn}
For the $(n+m)$-DOF HCDR shown in \autoref{fig:J1_GeneralizedHCDR}, the $j$th ($\{\forall \; j \in{\mathbb{N}} :0 \le j \le m\}$) COM (of the link) position vector $p_{acj}$ and joint position vector $p_{aj}$ are computed as
\begin{small}
\begin{align}
{p_{acj}} = &{[{p_{mx}},{p_{my}},{p_{mz}}]^T} + \underbrace {R_g^mR_m^{a0}}_{=:R_g^{a0}}{[{x_{a00}},{y_{a00}},{z_{a00}}]^T} \nonumber \\ 
 &+ \underbrace {R_g^mR_m^{a0}R_{a0}^{a1}}_{=:R_g^{a1}}{[{x_{a01}},{y_{a01}},{z_{a01}}]^T}  +  \cdots \nonumber \\
&+ \underbrace {R_g^mR_m^{a0}R_{a0}^{a1}R_{a1}^{a2} \cdots R_{a(j - 1)}^{aj}}_{=:R_g^{aj}}{[{x_{ac0j}},{y_{ac0j}},{z_{ac0j}}]^T} \nonumber \\ 
 = &{[{p_{mx}},{p_{my}},{p_{mz}}]^T} + R_g^{a0}{[{x_{a00}},{y_{a00}},{z_{a00}}]^T} \nonumber  + \nonumber \\  
&R_g^{a1}{[{x_{a01}},{y_{a01}},{z_{a01}}]^T}  + \cdots+ R_g^{aj}{[{x_{ac0j}},{y_{ac0j}},{z_{ac0j}}]^T} \nonumber \\
= &{[{p_{mx}},{p_{my}},{p_{mz}}]^T} + R_g^{aj}{[{x_{ac0j}},{y_{ac0j}},{z_{ac0j}}]^T} + \nonumber \\
 &\sum\limits_{k = 0}^j {\left\{ {R_g^{a(j - 1)}{{[{x_{a0(j - 1)}},{y_{a0(j - 1)}},{z_{a0(j - 1)}}]}^T}} \right\}}
 \label{eq:J1_10}
\end{align}
\end{small}
and
\begin{align}
{p_{aj}} = &{[{p_{mx}},{p_{my}},{p_{mz}}]^T} + R_g^mR_m^{a0}{[{x_{a00}},{y_{a00}},{z_{a00}}]^T}\nonumber \\
&+ R_g^mR_m^{a0}R_{a0}^{a1}{[{x_{a01}},{y_{a01}},{z_{a01}}]^T} +  \cdots \nonumber \\
& + R_g^mR_m^{a0}R_{a0}^{a1}R_{a1}^{a2} \cdots R_{a(j - 1)}^{aj}{[{x_{a0j}},{y_{a0j}},{z_{a0j}}]^T}\nonumber \\
 = &{[{p_{mx}},{p_{my}},{p_{mz}}]^T} + R_g^{a0}{[{x_{a00}},{y_{a00}},{z_{a00}}]^T} + \nonumber \\
&R_g^{a1}{[{x_{a01}},{y_{a01}},{z_{a01}}]^T}
+ \cdots  + R_g^{aj}{[{x_{a0j}},{y_{a0j}},{z_{a0j}}]^T}\nonumber \\
 = &{[{p_{mx}},{p_{my}},{p_{mz}}]^T} + \sum\limits_{k = 0}^j {\left\{ {R_g^{aj}{{[{x_{a0j}},{y_{a0j}},{z_{a0j}}]}^T}} \right\}} 
\label{eq:J1_11}
\end{align}
where 
${m_{aj}} \in\mathbb{R}$ is the mass of link $j$ and ${{{I}}_{aj}} \in {\mathbb{R}^{3 \times 3}}$ denotes the moment of inertia of link $j$. $[x_{ac0j},y_{ac0j},z_{ac0j}]^T$ and $[x_{a0j},y_{a0j},z_{a0j}]^T$ are 
body-fixed positions of the $j$th COM and joint, respectively. Also, for the $j$th revolute joint, a rotation matrix from frame $j-1$ to $j$ is defined as
\begin{align}
\begin{small}
R_{a(j - 1)}^{aj} = 
\begin{cases}
{R_x}({\theta _{aj}}), &\text{revolute joint about $X$-axis}\\
{R_y}({\theta _{aj}}), &\text{revolute joint about $Y$-axis}\\
{R_z}({\theta _{aj}}), &\text{revolute joint about $Z$-axis}
\end{cases}
\end{small}
\label{eq:J1_12}
\end{align}
and for the $j$th prismatic joint, the corresponding parameters of the revolute joint are replaced with 
\begin{align}
\begin{small}
\begin{cases}
R_{a(j - 1)}^{aj} = {I_{3 \times 3}}\\
{x_{ac0j}} = {x_{ac0j}} + {\theta _{aj}}, &\text{prismatic joint about $X$-axis}\\
{x_{a0j}} = {x_{a0j}} + {\theta _{aj}}, &\text{prismatic joint about $X$-axis}\\
{y_{ac0j}} = {y_{ac0j}} + {\theta _{aj}}, &\text{prismatic joint about $Y$-axis}\\
{y_{a0j}} = {y_{a0j}} + {\theta _{aj}}, &\text{prismatic joint about $Y$-axis}\\
{z_{ac0j}} = {z_{ac0j}} + {\theta _{aj}}, &\text{prismatic joint about $Z$-axis}\\
{z_{a0j}} = {z_{a0j}} + {\theta _{aj}}, &\text{prismatic joint about $Z$-axis}.
\end{cases}
\end{small}
\label{eq:J1_13}
\end{align}

The linear velocities of the $j$th COM (of the link) and joint are the time-derivative of positions in \eqref{eq:J1_10} and \eqref{eq:J1_11}, respectively. Then, we get
\begin{align}
{v_{acj}} &= {{\dot p}_{acj}} \label{eq:J1_14} 
\\
{v_{aj}} &= {{\dot p}_{aj}}. \label{eq:J1_15}
\end{align}

Additionally, the $j$th angle velocities are computed as
\begin{align}
{\omega _{acj}} = {\omega _{aj}} = {[R_m^{a0}R_{a0}^{aj}]^T}{\omega _m} + \sum\limits_{k = 0}^j {\left\{ {{{[R_{ak}^{aj}]}^T}{{\vec {\dot \theta}}_{ak}}} \right\}}
\label{eq:J1_16}
\end{align}
where ${{\vec {\dot \theta}}_{ak}} \in {\mathbb{R}^3}$ represents the vector of joint velocity about its body-fixed axis. \eqref{eq:J1_16} is a simplified and very useful result for generalized dynamic modeling, e.g., calculating the kinetic energy.

{\color{black}
\begin{lemm} \label{lemm:jth_angle_vel}
Let ${{\vec {\dot \theta}}_{ak}} \in {\mathbb{R}^3}$ be the vector of joint velocity about its body-fixed axis. Then the $j$th angle velocity vector is equal to ${\omega _{acj}} = {[R_m^{a0}R_{a0}^{aj}]^T}{\omega _m} + \sum\limits_{k = 0}^j {\left\{ {{{[R_{ak}^{aj}]}^T}{{\vec {\dot \theta}}_{ak}}} \right\}}$, where $\{\forall \; j,k \in{\mathbb{N}}:0 \le j \le m, \; 0 \le k \le j\}$.
\end{lemm}

\begin{proof}The $j$th angle velocity vector ${\omega _{acj}}$ can be derived as follows:
\begin{align*}
&{\omega _{acj}}= \nonumber \\
&{\quad}{[\underbrace {R_g^{m\alpha }R_{m\alpha }^{m\beta }R_{m\beta }^{m\gamma }}_{R_g^m}R_m^{a0}\underbrace {R_{a0}^{a1}R_{a1}^{a2} \cdots R_{a(j - 1)}^{aj}}_{R_{a0}^{aj}}]^T}{[{{\dot \alpha }_m},0,0]^T}\nonumber \\
&{\quad} + {[R_{m\alpha }^{m\beta }R_{m\beta }^{m\gamma }R_m^{a0}R_{a0}^{a1}R_{a1}^{a2} \cdots R_{a(j - 1)}^{aj}]^T}{[0,{{\dot \beta }_m},0]^T}\nonumber \\
&{\quad} + {[R_{m\beta }^{m\gamma }R_m^{a0}R_{a0}^{a1}R_{a1}^{a2} \cdots R_{a(j - 1)}^{aj}]^T}{[0,0,{{\dot \gamma }_m}]^T}+\nonumber \\
&{\quad} {[\underbrace {R_{a0}^{a1}R_{a1}^{a2} \cdots R_{a(j - 1)}^{aj}}_{R_{a0}^{aj}}]^T}{{\vec {\dot \theta}}_{a1}} + {[\underbrace {R_{a1}^{a2} \cdots R_{a(j - 1)}^{aj}}_{R_{a1}^{aj}}]^T}{{\vec {\dot  \theta}}_{a2}}\nonumber \\
&{\quad} + \cdots + {[\underbrace {R_{a(j - 2)}^{a(j - 1)}R_{a(j - 1)}^{aj}}_{R_{a(j - 2)}^{aj}}]^T}{{\vec {\dot \theta}}_{a(j - 1)}} + {[R_{a(j - 1)}^{aj}]^T}{{\vec {\dot \theta}}_{0j}}\nonumber \\
&= {[R_g^mR_m^{a0}R_{a0}^{aj}]^T}{[{{\dot \alpha }_m},0,0]^T} + {[R_{m\alpha }^{m\beta }R_{m\beta }^{m\gamma }R_m^{a0}R_{a0}^{aj}]^T}\nonumber \\
&{\quad} {[0,{{\dot \beta }_m},0]^T} + {[R_{m\beta }^{m\gamma }R_m^{a0}R_{a0}^{aj}]^T}{[0,0,{{\dot \gamma }_m}]^T} + {[R_{a0}^{aj}]^T}{{\vec {\dot \theta}} _{a1}} \nonumber \\
&{\quad}+ \cdots  + {[R_{a(j - 2)}^{aj}]^T}{{\vec {\dot \theta}} _{a(j - 1)}}  + {[R_{a(j - 1)}^{aj}]^T}{{\vec {\dot \theta}}_{aj}}\nonumber \\
&= {[R_m^{a0}R_{a0}^{aj}]^T}{[R_g^m]^T}{[{{\dot \alpha }_m},0,0]^T} + {[R_m^{a0}R_{a0}^{aj}]^T}{[R_{m\beta }^{m\gamma }]^T} \nonumber \\
&{\quad}{[R_{m\alpha }^{m\beta }]^T}{[0,{{\dot \beta }_m},0]^T}
 + {[R_m^{a0}R_{a0}^{aj}]^T}{[R_{m\beta }^{m\gamma }]^T}{[0,0,{{\dot \gamma }_m}]^T} \nonumber \\
&{\quad} + {[R_{a0}^{aj}]^T}{{\vec {\dot \theta}} _{a1}} +  \cdots  + {[R_{a(j - 2)}^{aj}]^T}{{\vec {\dot \theta}} _{a(j - 1)}} + {[R_{a(j - 1)}^{aj}]^T}{{\vec {\dot \theta}} _{aj}}\nonumber \\
&= {[R_m^{a0}R_{a0}^{aj}]^T}\underbrace {\left\{ \begin{array}{l}
{[R_g^m]^T}{[{{\dot \alpha }_m},0,0]^T} + {[R_{m\beta }^{m\gamma }]^T}{[R_{m\alpha }^{m\beta }]^T} \nonumber \\
{[0,{{\dot \beta }_m},0]^T} + {[R_{m\beta }^{m\gamma }]^T}{[0,0,{{\dot \gamma }_m}]^T}
\end{array} \right\}}_{{\omega _m}} \nonumber \\
&{\quad} + {[R_{a0}^{aj}]^T}{{\vec {\dot \theta}} _{a1}} +  \cdots + {[R_{a(j - 2)}^{aj}]^T}{{\vec {\dot \theta}} _{a(j - 1)}}
 + {[R_{a(j - 1)}^{aj}]^T}{{\vec {\dot \theta}} _{aj}} \nonumber \\
&= {[R_m^{a0}R_{a0}^{aj}]^T}{\omega _m} + \sum\limits_{k = 0}^j {\left\{ {{{[R_{ak}^{aj}]}^T}{{\vec {\dot \theta}} _{ak}}} \right\}} 
\end{align*}
where $\{\forall \; j,k \in{\mathbb{N}}:0 \le j \le m, \; 0 \le k \le j\}$, $R_g^{m\alpha }=R_{x}({\alpha _m}) \in {\mathcal{SO}(3)}$, $R_{m\alpha}^{m\beta}=R_{y'}({\beta _m}) \in {\mathcal{SO}(3)}$, and $R_{m\beta}^{m\gamma}=R_{z''}({\gamma _m}) \in {\mathcal{SO}(3)}$. \qedhere
\end{proof}}

Substituting the corresponding results in \eqref{eq:J1_14} and \eqref{eq:J1_16}, the total kinetic energy is calculated as
\begin{align}
{K_E} = &\frac{1}{2}{m_m}[{{\dot p}_{mx}},{{\dot p}_{my}},{{\dot p}_{mz}}]{[{{\dot p}_{mx}},{{\dot p}_{my}},{{\dot p}_{mz}}]^T} + \frac{1}{2}\omega _m^T{I_m}{\omega _m}\nonumber \\
&+ \frac{1}{2}\sum\limits_{k = 0}^j {\left\{ {{m_{ak}}v_{ack}^T{v_{ack}} + \omega _{ack}^T{I_{ak}}{\omega _{ack}}} \right\}}.
\label{eq:J1_17}
\end{align}

The total potential energy is computed as
\begin{align} 
{V_E} = &{m_m}g{p_{mz}} + \sum\limits_{k = 0}^j {\left\{ {{m_{ak}}g{p_{ack}^T}{{[0,0,1]}^T}} \right\}} \nonumber \\
&+\frac{1}{2}{\left( {L - {L_0}} \right)^T}{K_c}\left( {L - {L_0}} \right)
\label{eq:J1_18}
\end{align}
where $g$ represents the gravity acceleration, position vector $p_{ack}$ is obtained using \eqref{eq:J1_10}, and $\frac{1}{2}{\left( {L - {L_0}} \right)^T}{K_c}\left( {L - {L_0}} \right)$ denotes the cable elastic potential energy with its variables defined in \eqref{eq:J1_8}.

Based on the computed kinetic energy ${K_E}$ and potential energy ${V_E}$ in \eqref{eq:J1_17} and \eqref{eq:J1_18}, respectively, the Lagrangian dynamic equation is calculated as
\begin{align}
{L_E} = {K_E} - {V_E}.
\label{eq:J1_19}
\end{align}

Then, the torque equations are calculated as
\begin{align}
{\tau_j} &= \frac{d}{{dt}}\left( {\frac{{\partial {L_E}}}{{\partial {{\dot q}_j}}}} \right) - \frac{{\partial {L_E}}}{{\partial {q_j}}} \nonumber\\
 &= \frac{d}{{dt}}\left( {\frac{{\partial {K_E}}}{{\partial {{\dot q}_j}}}} \right) - \frac{{\partial {K_E}}}{{\partial {q_j}}} + \frac{{\partial {V_E}}}{{\partial {q_j}}}\quad \label{eq:J1_20}
\end{align}
where ${\tau_j}$ represents the generalized force/torque applied to the dynamic system at joint $j$ to drive link $j$. 

Based on open-chain, \eqref{eq:J1_20} can be described by a new form:
\begin{align}
{\tau _j} = \scriptstyle \bigg\{
{{\left[{\nabla {{\left( {{{\left( {\nabla {{L_E}_{\dot q}}} \right)}_j}} \right)}_q}} \right]}^T}\dot q + {{\left[ {\nabla {{\left( {{{\left( {\nabla {{L_E}_{\dot q}}} \right)}_j}} \right)}_{\dot q}}} \right]}^T}\ddot q
- {{\left( {\nabla {{L_E}_q}} \right)}_j}\bigg\}
\label{eq:J1_jDOFLagrangian}
\end{align}
where $\nabla {\left(\cdot \right)_q}$ and $\nabla{\left(  \cdot \right)_{\dot q}}$ are defined as the gradient vectors of $(\cdot)$ with respect to the vectors $q$ and $\dot q$, respectively. Compared with \eqref{eq:J1_20},  \eqref{eq:J1_jDOFLagrangian} is easier to be implemented (i.e., programming). By arranging \eqref{eq:J1_jDOFLagrangian} and introducing a vector of unknown bounded disturbances ${\tau_d} \in {\mathbb{R}^{n + m}}$, the equations of motion of the HCDR can be derived as
\begin{align}
{{M}}\left( {{q}} \right){{\ddot q + C}}\left( {{{q}},{{\dot q}}} \right){{\dot q + G}}\left( {{q}} \right) + {\tau _d} = :\tau {{ = }}\left[ {\begin{array}{*{20}{c}}
{{\tau _m}}\\
{{\tau _{a}}}
\end{array}} \right]{{ = }}\left[ {\begin{array}{*{20}{c}}
{{A_mT}}\\
{{\tau _{a}}}
\end{array}} \right]
\label{eq:J1_22}
\end{align}
where $q \in {\mathbb{R}^{n + m}}$, $\dot q \in {\mathbb{R}^{n + m}}$, and $\ddot q \in {\mathbb{R}^{n + m}}$ represent the vectors of generalized coordinates, velocities, and accelerations, respectively. $M(q) \in {\mathbb{R}^{(n + m) \times (n + m)}}$ denotes the combined inertia matrix,  $C(q,\dot q) \in {\mathbb{R}^{(n + m) \times (n + m)}}$ represents the combined Coriolis and centripetal matrix, and $G(q) \in {\mathbb{R}^{n + m}}$ denote the gravitational vector, respectively. ${\tau _d} \in {\mathbb{R}^{n + m}}$ and $\tau  \in {\mathbb{R}^{n + m}}$ denote the vector of unknown bounded disturbances and forces/torques in generalized coordinates, respectively. Eq. \eqref{eq:J1_22} is the inverse dynamics model for HCDR, with $q$, $\dot q$, and $\ddot q$ are inputs. 
\begin{prop} \label{prop:J1_1}
$M$ is a symmetric and positive definite matrix \cite{Lin2001}
\begin{align}
\dot M = C + {C^T}.
\label{eq:J1_23}
\end{align}
\end{prop}

Since the inertia matrix  $M$ is symmetric and positive definite, then the forward dynamics can be computed as
\begin{align}
{{\ddot q}} = {{{M}}^{ - 1}}\left( {{q}} \right)\left( {\left[ {\begin{array}{*{20}{c}}
{{{A_mT}}}\\
{{\tau _{a}}}
\end{array}} \right] - {{C}}\left( {{{q}},{{\dot q}}} \right){{\dot q}} - {{G}}\left( {{q}} \right) - {\tau _d}} \right)
\label{eq:J1_24}
\end{align}
where the cable tension $T$ and robot arm joint torque ${\tau _{a}}$  are inputs. 

Additionally, \eqref{eq:J1_22} can be arranged as
\begin{align}
&\left[ {\begin{array}{*{20}{c}}
{{M_{11}}(q)}&{{M_{12}}(q)}\\
{{M_{21}}(q)}&{{M_{22}}(q)}
\end{array}} \right]\left[ {\begin{array}{*{20}{c}}
{{{\ddot q}_m}}\\
{{{\ddot q}_{a}}}
\end{array}} \right] + \left[ {\begin{array}{*{20}{c}}
{{C_{11}}(q,\dot q)}&{{C_{12}}(q,\dot q)}\\
{{C_{21}}(q,\dot q)}&{{C_{22}}(q,\dot q)}
\end{array}} \right]\nonumber\\
&\left[ {\begin{array}{*{20}{c}}
{{{\dot q}_m}}\\
{{{\dot q}_{a}}}
\end{array}} \right]
+ \left[ {\begin{array}{*{20}{c}}
{{G_m}({q_m})}\\
{{G_{a}}({q_{a}})}
\end{array}} \right] + \left[ {\begin{array}{*{20}{c}}
{{\tau _{md}}}\\
{{\tau _{ad}}}
\end{array}} \right] = \left[ {\begin{array}{*{20}{c}}
{{{{A_m}T}}}\\
{{\tau_{a}}}
\end{array}} \right]
\label{eq:J1_25}
\end{align}
where ${( \cdot )_m} \in {\mathbb{R}^n}$ and ${( \cdot )_{a}} \in {\mathbb{R}^m}$ represent the vector of the mobile platform variables and the robot arm variables, respectively. It is clear that this equation includes the dynamics of the CDPR and the mounted robot arm.

In summary, some key features of the proposed modeling method can be highlighted as follows: 1) {\color{black}The derivation of the equations of motion (e.g.,{~}{\eqref{eq:J1_10}}--{\eqref{eq:J1_16}} and{~}{\eqref{eq:J1_jDOFLagrangian}}) and the proof of{~}{\eqref{eq:J1_16}} provide a very effective way to find items for generalized system modeling.} 
2) Traditionally, based on the rule of Standard Denavit-Hartenberg (DH) parameters \cite{Denavit1955}, a revolute joint must rotate about its $Z$-axis. Sometimes, it is inconvenient or impossible to find DH parameters (e.g., rotating around the $X$-axis or $Y$-axis). {\color{black}Moreover, it is far easier to configure DH parameters for complex mechanisms~\cite{R.Balasubramanian2011}.} The proposed method avoids this drawback, i.e., it is unnecessary to follow DH configurations, and can be applied to any coordinate frames (e.g., rotating around the $X$-axis, $Y$-axis, or $Z$-axis), including revolute and prismatic joints. 
3) The above modeling approach in this section can be easily extended to other types of hybrid robots by changing structure matrix $A_m$ in~\eqref{eq:J1_5} based on their configurations, e.g., a robot arm(s) mounted on an aircraft platform~\cite{Tardella2016,PRODRONE2016}. {\color{black}To illustrate the extension approach, a detailed example is provided (see Appendix~\ref{appendix:J1_DroneArm_Dynamics}). In this case, the DH configuration is also nonessential.}

\subsection{Redundancy Resolution}
Cable-driven robots (as shown in \autoref{fig:J1_GeneralizedHCDR}) can be categorized into under-actuated, fully-actuated, and over-actuated \cite{Corke2011}. The first two types of robots denote the number of driven cables $N$ is no more than the DOF of a robot $n$, i.e., $N \le n$; the third type of robots represents the number of driven cables $N$ is more than the DOF of a robot n, i.e., $N > n$. Then, the value of $(N-n)$ represents the degree of redundancy (DOR). When redundancy problems exist, there are infinite solutions for kinematics, which make the motion planning challenging~\cite{Taghirad2011}. Usually, redundancy resolution (i.e., over-actuated) problems are more general for cable-driven robots and can be solved only using the pseudo-inverse approach{~}\cite{Corke2011}, but the solutions are not optimal. {\color{black}Some other approaches are also available, such as a combination of pseudo-inverse and null-space method{~}{\cite{Corke2011,Mendez2014,Behzadipour2006}}, damped least-squares approach{~}{\cite{J.Li2017}}, and energy-based method{~}{\cite{J.Li2017}}.} In this paper, we use the combined method{~}\cite{Corke2011,Mendez2014,Behzadipour2006} to address the redundancy resolution problem.

When $q_m$, $\dot{q}_m$, and $\ddot{q}_m$ are given, ${\tau_m}={{A_mT}}$ can be computed using \eqref{eq:J1_9} or \eqref{eq:J1_22}. Then, the cable tension $T$ is calculate as
\begin{align}
{T} = {A_m^+}{\tau _m} = {A_m^T}{({{A_m}}{A_m^T})^{ - 1}}{\tau _m}
\label{eq:J1_26}
\end{align}
{\color{black}where ${A_m^+}$ represents the pseudo-inverse of matrix $A_m$ and ${A_m^+}$ is equal to ${A_m^T}{({A_m}{A_m^T})^{ - 1}}$. 

Singularity occurs when $A_m$ is not full rank or $\rm{det}(A_m)=0$. It leads HCDRs or CDPRs to be uncontrollable; hence, singularity should be avoided in applications. However, by given physical constraints, e.g., $\beta_m \ne \pm \pi/2$ for the configuration shown in~\autoref{fig:J1_9dofHCDPR}, the singularity of $A_m$ is not realistically reached. For a large matrix $A_m$, one can find a numerical solution to reduce the computation process of the pseudo-inverse of $A_m$.}

In \eqref{eq:J1_26}, the elements of the cable tension $T \in {\mathbb{R}^N}$ might be negative. However, in practice, they cannot drive the mobile platform if they are negative. The redundancy resolution of the cable tension $T$ can be formulated as
\begin{align}
{T} = {A_m^T}{({{A_m}}{A_m^T})^{ - 1}}{\tau _m} + {N_A}\lambda
\label{eq:J1_27}
\end{align}
where ${N_A} \in {\mathbb{R}^{N \times (N - n)}}$ represents the null space of structure matrix $A_m$ ($A_m$ is calculated using \eqref{eq:J1_5}), and $\lambda  \in {\mathbb{R}^{N - n}}$ is a vector of arbitrary values. In \eqref{eq:J1_27}, ${N_A}\lambda$ belongs to the null space of $A_m$, since it can be described as ${A_m}\left( {{N_A}\lambda } \right) = \left( {{A_m}{N_A}} \right)\lambda  = \textbf{0}$. The expression ${N_A}\lambda$ denotes antagonistic cable tensions. The cable tension $T$ increases if all the antagonistic cable tensions are positive. Hence, the vector $\lambda$ can be optimized (e.g., using the stiffness optimization method in the next section) to ensure that all the cable tensions are positive.

\subsection{Stiffness Optimization}
To solve the above problem of selecting $\lambda$, a stiffness maximization method is proposed as below: consider the same condition as \eqref{eq:J1_26}, the stiffness matrix $K$ is defined as
\begin{align}
{K}: = &{}\frac{{d({{A_mT}})}}{{d{{{P}}_m}}} = \frac{{d{{A_m}}}}{{d{{{P}}_m}}}{{T}} + {{A_m}}\frac{{d{{T}}}}{{d{{{P}}_m}}} = \frac{{d{{A_m}}}}{{d{{{P}}_m}}}{{T}} 
+ {{A_m}}\left( {\frac{{d{{T}}}}{{d{{L}}}}} \right)\nonumber\\
&{}\left( {\frac{{d{{L}}}}{{d{{{P}}_m}}}} \right) = \frac{{d{{A_m}}}}{{d{{{P}}_m}}}{{T}} + {{A_m}}{{{K}}_c}{A_m^T} = :{{{K}}_T} + {{{K}}_k}
\label{eq:J1_28}
\end{align}
where $P_m:=[{p_{mx}},{p_{my}},{p_{mz}},{\alpha _m},{\beta _m},{\gamma _m}]^T\in {\mathbb{R}
^6}$, $T$, and $L$ represent the position and orientation of the center of mass of the mobile platform, cable tension vector, and cable length vector, respectively. Matrices $K_T$ and $K_k$ are a product of the cable tensions and cable stiffness, respectively, where ${K_c} = \frac{{d{{T}}}}{{d{{L}}}} = {\color{black}{\rm{diag}}\left[ {{k_{c1}},{k_{c2}}, \cdots ,{k_{ci}}, \cdots ,{k_{cN}}} \right]} \in {\mathbb{R}
^{N \times N}}$ and $k_{ci}$ denotes the $i$th cable stiffness (same as \eqref{eq:J1_8}).

Usually, $K$ is obtained at static condition for easier stability analysis. In this case, if \eqref{eq:J1_28} is expanded in terms of the kinematic parameters $L_i$, ${\hat L}_i$, and ${r_i}$, the matrices $K_T$ and $K_k$ can be described as \cite{Behzadipour2006}
\begin{align}
{K_T} = &\sum\limits_{i = 1}^N {\frac{T_i}{L_i}\left[ {\begin{array}{*{20}{c}}
{{{I}} - {{{{\hat L}}}_i}{{\hat L}}_i^T}&{({{I}} - {{{{\hat L}}}_i}{{\hat L}}_i^T){{[{{{r}}_i}]}^T}}\\
{[{{{r}}_i}]({{I}} - {{{{\hat L}}}_i}{{\hat L}}_i^T)}&{[{{{r}}_i}]({{I}} - {{{{\hat L}}}_i}{{\hat L}}_i^T){{[{{{r}}_i}]}^T}}
\end{array}} \right]} \nonumber \\
&- \sum\limits_{i = 1}^N {{T_i}\left[ {\begin{array}{*{20}{c}}
\textbf{0}&\textbf{0}\\
\textbf{0}& [{{{{\hat L}}}_i}]{[{{{r}}_i}]}\end{array}}\right]}
\label{eq:J1_29}
\end{align}
and
\begin{align}
{K_k} = \sum\limits_{i = 1}^N {{k_{ci}}\left[ {\begin{array}{*{20}{c}}
{{{{{\hat L}}}_i}{{\hat L}}_i^T}&{{{{{\hat L}}}_i}{{\hat L}}_i^T{{[{{{r}}_i}]}^T}}\\
{[{{{r}}_i}]{{{{\hat L}}}_i}{{\hat L}}_i^T}&{[{{{r}}_i}]{{{{\hat L}}}_i}{{\hat L}}_i^T{{[{{{r}}_i}]}^T}}
\end{array}} \right]}
\label{eq:J1_30}
\end{align}
where 
\begin{small}
$[{r_i}]: = \left[ {\begin{array}{*{20}{c}}
0&{ - {{({R_g^m}{r_i})}_{3,1}}}&{{{({R_g^m}{r_i})}_{2,1}}}
\\
{{{({R_g^m}{r_i})}_{3,1}}}&0&{ - {{({R_g^m}{r_i})}_{1,1}}}
\\
{ - {{({R_g^m}{r_i})}_{2,1}}}&{{{({R_g^m}{r_i})}_{1,1}}}&0
\end{array}} \right]$
\end{small}, \begin{small}
$[{{\hat L}_i}]: = -\left[ {\begin{array}{*{20}{c}}
0&{ - {{\hat L}_{iz}}}&{{{\hat L}_{iy}}}\\
{{{\hat L}_{iz}}}&0&{ - {{\hat L}_{ix}}}\\
{ - {{\hat L}_{iy}}}&{{{\hat L}_{ix}}}&0
\end{array}} \right]$, ${r_i} = {[{r_{ix}},{r_{iy}},{r_{iz}}]^T}$
\end{small}, ${{\hat L}_i} = {[{\hat L_{ix}},{\hat L_{iy}},{\hat L_{iz}}]^T}$, and $I$ is the identity matrix. Eq. \eqref{eq:J1_29} and \eqref{eq:J1_30} are equivalent to the results of the four-spring model proposed by Behzadipour and Khajepour \cite{Behzadipour2006}. They proved that a static cable-driven robot is stable if the stiffness matrix $K$ is positive definite (sufficient condition). In addition, elements of $K_k$ cannot be controlled, because they are the property of the cables. Hence, the stiffness of HCDR can only be changed by optimizing $K_T$.

Additionally, combine \eqref{eq:J1_27} and \eqref{eq:J1_28}, we get
\begin{align}
{K}(\lambda ) = &\left( {\frac{{d{{A_m}}}}{{d{{{P}}_m}}}{N_A}} \right)\lambda  + \frac{{d{{A_m}}}}{{d{{{P}}_m}}}{A_m^T}{({A_m}{A_m^T})^{ - 1}}{\tau _m} \nonumber \\
&+ {{A_m}}{{{K}}_c}{A_m^T}.
\label{eq:J1_31}
\end{align}

Since $K$ is positive definite (or positive semidefinite), the maximum stiffness is determined by its eigenvalues \cite{M.Li2011,Kock1998,Gosselin1990}. Hence, the optimization problem can be described as
\begin{subequations} 
\label{eq:J1_32}
\begin{align}
&{\mathop{\rm{max}}\limits_\lambda} && J_K={\rm{eig}}{({{K}}(\lambda ))^T}{H_\lambda }{\rm{eig}}({{K}}(\lambda )) \label{eq:J1_32a}\\
&{\rm{s. \; t.}} && {{{M}}\left( {{q}} \right){{\ddot q + C}}\left( {{{q}},{{\dot q}}} \right){{\dot q + G}}\left( {{q}} \right) + {\tau _d}{{ = }}\left[ {\begin{array}{*{20}{c}}{{{A_mT}}}\\{{\tau_a}}\end{array}} \right]} \label{eq:J1_32b}\\
&&& {{K}(\lambda)} = {\left( {\frac{{d{{A_m}}}}{{d{{{P}}_m}}}{N_A}} \right)\lambda  + \frac{{d{{A_m}}}}{{d{{{P}}_m}}}{A_m^T}{({A_m}{A_m^T})^{ - 1}}{\tau _m}} \nonumber \\
&&&{+ {{A_m}}{{{K}}_c}{A_m^T}} \label{eq:J1_32c} \\
&&& 0 \le {T_{i\min }} \le {T_i} \le {T_{i\max }}, \;i = 1,2, \cdots ,N \label{eq:J1_32d}\\
&&&{\color{black} T = [T_1,T_2,\cdots,T_i\cdots,T_N]^T} \label{eq:J1_32e}
\end{align}
\end{subequations}
where $T_i$, ${T_{i\min }}$, and ${T_{i\max }}$ represent the $i$th cable tension, minimum and maximum allowable tensions, respectively. ${H_\lambda }$ denotes the stiffness weighting matrix. To ensure the stability of a HCDR in practical applications, some alternative strategies can be adopted as below: 1) optimizing its trajectory to keep all the eigenvalues of $K$ positive and 2) limiting the maximum payload \cite{Behzadipour2006}. Comparing with the existed stiffness optimization approaches in{~}\cite{Khajepour2015,Mendez2014,H.Jamshidifar2017, H.Jamshidifar2018}, \eqref{eq:J1_32a} is introduced by combining the eigenvalues of $K$ and weighting matrix ${H_\lambda}$ so that one is able to optimize the system stiffness based on specific needs (by tuning ${H_\lambda}$). Meanwhile, \eqref{eq:J1_32} can be applied to not only CDPRs but also HCDRs.
\begin{figure}[h]\centering
	\includegraphics[width=88mm]{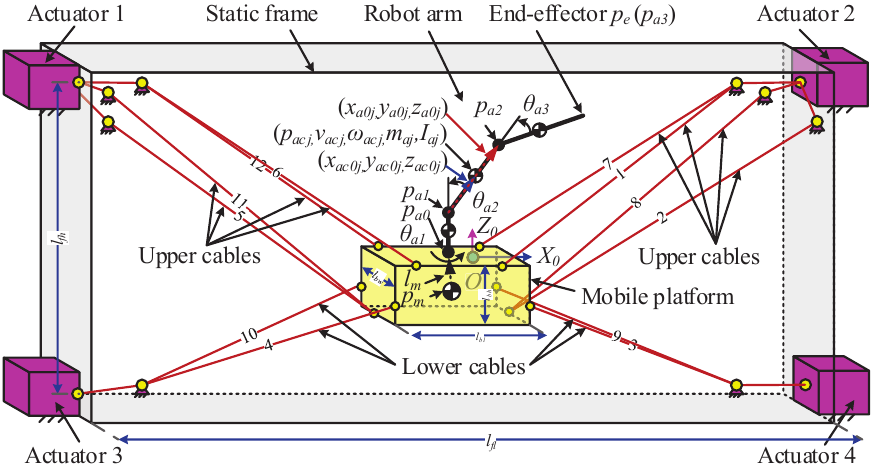}
	\caption{Configuration of the 9-DOF HCDR. The CDPR is driven by four actuators with 12 cables; the robot arm has three joints with the first, second, and third joints rotating about $Z_{a0}$-, $Y_{a1}$-, and $Y_{a2}$-axis (i.e., the corresponding moving frames), respectively.}
	\label{fig:J1_9dofHCDPR}
\end{figure}

\section{HCDR Example for Analysis}\label{sec:J1_HCDPRExampleforAnalysis}
\subsection{HCDR Configuration and Kinematic Constraints}
\label{subsec:J1_HCDPRConfig}
To analyze the generalized HCDR model in \autoref{sec:J1_GeneralizedSystemModeling} and study its control performance, a HCDR example is given in \autoref{fig:J1_9dofHCDPR}. The proposed full model of the HCDR has 9 DOFs (the CDPR parameters come from \cite{Mendez2014,Rushton2016,Rushton2018}), which consists of a 3-DOF robot arm (i.e., $m=3$), a 6-DOF CDPR (i.e., $n=6$), twelve cables, and four servo motors. The 3-DOF robot arm has three revolute joints, rotating about $Z$-axis (joint 1 frame), $Y$-axis (joint 2 frame), and $Y$-axis (joint 3 frame), respectively. The actuators are used to drive the cables to move the mobile platform. The robot arm is fixed on the mobile platform and moves with it. The twelve cables include four sets of cables: two sets of four-cable arrangement on the top and two sets of two-cable arrangement on the bottom. The driven cable mount locations the HCDR are shown in \autoref{table:J1_CableMountLocations}. Each set of cables is controlled by one motor. In addition, the top actuators and bottom actuators control the upper cable lengths and lower cable tensions, respectively. The upper cables also restrict the orientation of the mobile platform, i.e., the kinematic constraints. In addition, the inertial coordinate frame $O\left\{ {{x_0},{y_0},{z_0}} \right\}$ is located at the center of the static fixture.

Additionally, other HCDR parameters are shown in \autoref{table:J1_HCDPRParameters}, where $m_m$ and $I_m$ represent the mass and moment of inertia of the mobile platform, respectively. $m_{aj}$ and $I_{aj}$ ($\{\forall \; j \in{\mathbb{N}} :1 \le j \le 3\}$) respectively denote the mass and moment of inertia of robot arm links. Also, $T_{i\min}$ and $T_{i\max}$ ($\{\forall \; i \in{\mathbb{N}} :1 \le i \le 12\}$) represent the minimum and maximum allowable cable tensions, respectively. The sizes of the static fixture (e.g., $l_{fl}$) and mobile platform (e.g., $l_{bl}$), body-fixed positions (e.g., $[x_{a0j},y_{a0j},z_{a0j}]^T$), and etc. are also given in \autoref{table:J1_HCDPRParameters}.
\begin{table}[h]
	\renewcommand{\arraystretch}{1.3}
	\caption{Driven Cable Mount Locations}
	\centering
	\label{table:J1_CableMountLocations}
	\resizebox{\columnwidth}{!}{
		\begin{tabular}{c c c c c c c}
			\hline\hline \\[-3mm]
$N$ & $a_{ix} \;\rm{(m)}$ & $a_{iy} \;\rm{(m)}$ & $a_{iz} \;\rm{(m)}$ & 
$r_{ix} \;\rm{(m)}$ & $r_{iy} \;\rm{(m)}$ & $r_{iz} \;\rm{(m)}$  \\[1.6ex] \hline
1 & 1.500 & 0.000 & 0.500 & 0.153 & -0.065 & 0.048\\
2 & 1.580 & -0.065 & 0.404 & 0.233 & 0.000 & -0.048\\
3 & 1.500 & 0.000 & -0.500 & 0.223 & -0.088 & -0.017\\
4 & -1.500 & 0.000 & -0.500 & -0.223 & -0.088 & -0.017\\
5 & -1.580 & -0.065 & 0.404 & -0.233 & 0.000 & -0.048\\
6 & -1.500 & 0.000 & 0.500 & -0.153 & -0.065 & 0.048\\
7 & 1.500 & 0.000 & 0.500 & 0.153 & 0.065 & 0.048\\
8 & 1.580 & 0.065 & 0.404 & 0.233 & 0.000 & -0.048\\
9 & 1.500 & 0.000 & -0.500 & 0.223 & 0.088 & -0.017\\
10 & -1.500 & 0.000 & -0.500 & -0.223 & 0.088 & -0.017\\
11 & -1.580 & 0.065 & 0.404 & -0.233 & 0.000 & -0.048\\
12 & -1.500 & 0.000 & 0.500 & -0.153 & 0.065 & 0.048\\
		\hline\hline
		\end{tabular}
	}
\end{table}
\begin{table}[h]
	\renewcommand{\arraystretch}{1.3}	
	\caption{HCDR Parameters}
	\centering
	\label{table:J1_HCDPRParameters}
	\resizebox{\columnwidth}{!}{
		\begin{tabular}{c c c c}
			\hline\hline \\[-3mm]
Symbol & Values & Symbol & Values  \\[1.6ex] \hline
$l_{fl}$ & $3.160\;\rm{m}$ & $l_{fh}$ & $1.000\;\rm{m}$\\
$l_{bl}$ & $0.365\;\rm{m}$ & $l_{bw}$ & $0.130\;\rm{m}$ \\
$l_{bh}$ & $0.096\;\rm{m}$ & $l_{m}$  & $[0,0,0.048]^T\;\rm{m}$\\
$m_{aj}$ & $0.400\;\rm{kg}$ & $I_{aj}$ & ${\color{black}\rm{diag}[0.100,0.100, 0.100]}\;\rm{kg \cdot m^2}$\\
$[x_{ac0j},y_{ac0j},z_{ac0j}]^T$ & $[0,0,0.050]^T\;\rm{m}$ & $[x_{a0j},y_{a0j},z_{a0j}]^T$ & $[0,0,0.100]^T\;\rm{m}$\\
$m_m$ & $10.000 \;\rm{kg}$ & $I_m$ & ${\color{black}\rm{diag}[0.0218,0.1187,0.1251]}\;\rm{kg \cdot m^2}$\\
$EA_i$ & $100 \;\rm{N}$ & $[T_{i\min},\;T_{i\max}]$ & $[5, 80] \;\rm{N}$\\
$g$ & $9.810 \;\rm{m/s^2}$ & &\\
		\hline\hline
		\end{tabular}
	}
\vspace{-0.2cm}
\end{table}

\subsection{Dynamics of the 9-DOF HCDR}
\label{subsec:J1_9DOFHCDPRDyn}
By applying the modeling method in \autoref{sec:J1_GeneralizedSystemModeling}, the detailed motion of equations can be computed for the specific 9-DOF system (see Appendix \ref{appendix:J1_9DOFDynamics}), where $q = [{p_{mx}},{p_{my}},{p_{mz}},{\alpha _m},{\beta _m},{\gamma _m},{\theta _{a1}},{\theta _{a2}},{\theta _{a3}}]_{9 \times 1}^T \in {\mathbb{R}^{6 + 3}}$, ${{\dot q}} \in \mathbb{R}{^{6 + 3}},{{\ddot q}} \in \mathbb{R}{^{6 + 3}},{{M}}\left( {{q}} \right) \in \mathbb{R}{^{(6 + 3) \times (6 + 3)}},{{C}}\left( {{{q}},{{\dot q}}} \right) \in \mathbb{R}{^{(6 + 3) \times (6 + 3)}}$, ${{G}}\left( {{q}} \right) \in \mathbb{R}{^{6 + 3}},{\tau _d} \in \mathbb{R}{^{6 + 3}},{\tau _m} \in \mathbb{R}{^6},{\tau _{a}} \in \mathbb{R}{^3},{{A_m}} \in \mathbb{R}{^{6 \times 12}}$, and ${{T}} \in \mathbb{R}{^{12}}$. However, because of the kinematic constraints, the system is fully controllable in $x_0z_0$ plane, then the 9-DOF HCDR is simplified as a 5-DOF in-plane system. The new control inputs are defined as ${u}: =(u_m,u_a):= [{{T_3},{T_4},{\tau _{a2}},{\tau _{a3}}]^T} \in {\mathbb{R}^4}$, where $u_m=[{T_3},{T_4}]^T \in {\mathbb{R}^2}$ denote the lower cable tensions (two sets of two-cable arrangement on the bottom), i.e., ${T_3}$ (driven by actuator 3) represents cable tensions 4 and 10, and ${T_4}$ (driven by actuator 4) denotes cable tensions 3 and 9 (shown in \autoref{fig:J1_9dofHCDPR}). $u_a={\tau_{a}}= {[{\tau _{a2}},{\tau _{a3}}]^T} \in {\mathbb{R}^2}$ represent input torques corresponding the second and the third joints of the robot arm. For the simplified 5-DOF HCDR, the CDPR has a more number of actuators (4 actuators) than the total DOFs (3 DOFs), and the robot arm has an equal number of actuators to its total DOFs (2 DOFs), so they are over-actuated and fully-actuated subsystems, respectively. Hence, \eqref{eq:J1_24} can be expressed as
\begin{align}
\dot x(t) = f(x(t),u(t),(L_{01}(t),L_{02}(t))), \; x(0) = {x_0}
\label{eq:J1_33}
\end{align}
where $x:= [{p_{mx}},{\dot p_{mx}},{p_{mz}},{\dot p_{mz}},{\beta _m},{{\dot \beta }_m},{\theta _{a2}},{{\dot \theta }_{a2}},{\theta _{a3}},{{\dot \theta }_{a3}}{]^T}$ $\in {\mathbb{R}^{10}}$ represents the states, $u \in {\mathbb{R}^4}$ denotes the control inputs, ${L_{01}}$ and ${L_{02}}$ represent the upper unstretched cable lengths (the two sets of four-cable arrangement on the top), respectively, i.e., ${L_{01}}$ (driven by actuator 1) denotes unstretched cable lengths 5, 6, 11, and 12, and ${L_{02}}$ (driven by actuator 2) represents unstretched cable lengths 1, 2, 7, and 8 (as shown in \autoref{fig:J1_9dofHCDPR}). ${x_0} \in {\mathbb{R}^{10}}$ is the initial states and $t \ge 0$.

By linearizing the nonlinear \eqref{eq:J1_33} around the reference states $x_r$ and control inputs $u_r$, the continuous time state-space representation (Linear Time-Varying System (LTV)) can be described as
\begin{align}
\begin{array}{l}
\dot x(t) = A(t)x(t) + B(t)u(t) + B(t)w(t)\\
y(t) = C(t)x(t), \; x(0) = {x_0}
\end{array}
\label{eq:J1_34}
\end{align}
with the outputs $y(t) \in {\mathbb{R}^{10}}$, matrices $A(t) = {\left. {\frac{{\partial f(x,u)}}{{\partial x}}} \right|_{x = {x_r},u = {u_r}}} \in {\mathbb{R}^{10 \times 10}}$, $B(t) = {\left. {\frac{{\partial f(x,u)}}{{\partial u}}} \right|_{x = {x_r},u = {u_r}}} \in {\mathbb{R}^{10 \times 4}}$, and $C(t) = I  \in {\mathbb{R}^{10 \times 10}}$. $w(t)=(w_m(t),w_a(t)) \in {\mathbb{R}^{4}}$ are white noises with zero mean Gaussian, where $w_m(t) \in {\mathbb{R}^{2}}$ and $w_a(t) \in {\mathbb{R}^{2}}$ are noises to the CDPR and robot arm, respectively. 

Additionally, for the specific HCDR, the upper four cables are utilized for position control and the lower cables are used to set cable tensions. Hence, the specific stiffness matrix \eqref{eq:J1_29} and \eqref{eq:J1_30} can be rearranged as $K_T = \sum\limits_{i = 1}^{12} {( \cdot )}$ and $K_k = \sum\limits_{i \in \{1,2,5,6,7,8,11,12\}}{(\cdot)}$, respectively. The actual states of the mobile platform and robot arm can be estimated using an Inertial Measurement Unit (IMU) (or vision-based tracking system) and encoders, respectively.

\section{Vibration Control Design}
\label{sec:J1_VibCtrlDesign}
For the proposed HCDR, a key objective is to develop effective control schemes to minimize vibration and improve the accuracy of the end-effector, i.e., the position-holding performance of the CDPR and the position accuracy performance of the end-effector relative to its reference trajectory.

For the configuration of the HCDR shown in Section III, the eight upper cables, four lower cables, and robot arm are based on position control, force control, and torque control, respectively, i.e., their corresponding inputs are positions (cable lengths), forces, and joint torques. Furthermore, the flexible cables reduce the overall stiffness of the robot, so vibrations become a serious problem for precise control \cite{Behzadipour2006}. Another major problem is maintaining cable tensions to keep large enough stiffness for the robot, as mentioned above. Besides, since the driven cables are flexible, the positions of the mobile platform or actual cable lengths cannot be computed directly from the measurements of encoders (embedded in the corresponding driven actuators). However, the upper unstretched (i.e., nominal) cable lengths ($L_{01}$ and $L_{02}$ in \eqref{eq:J1_33}) can be obtained using \eqref{eq:J1_33} when the reference trajectory $r(t)=x_r$ is given. Then, we can readjust lower cable tensions using \eqref{eq:J1_32}. Hence, tracking the reference trajectory as well as optimizing the lower cable tensions to satisfy the required stiffness of the HCDR should be included in the control design.

Based on the above analysis, in order to achieve the above objective, the proposed control structures of the HCDR are shown in \autoref{fig:J1_ThreeCtrlStructure}(a-c). Additionally, because of the kinematic constraints, the system \eqref{eq:J1_34} is fully controllable in $x_0z_0$ plane. The states and control inputs are simplified for the proposed vibration control.
\begin{figure}[h]\centering
	\includegraphics[width=88mm]{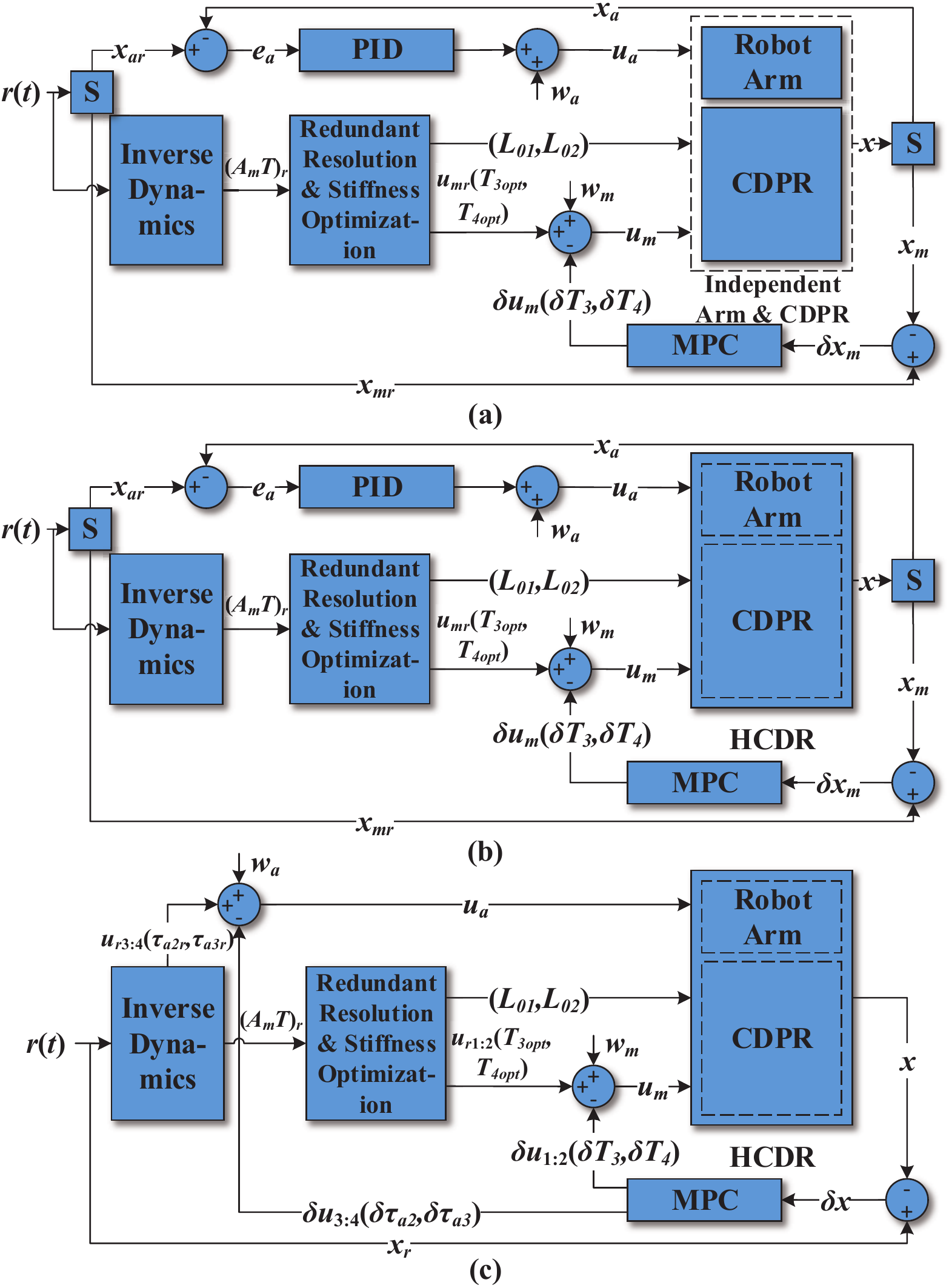}
	\caption{Three types of control architecture of the HCDR. (a) Independent control with the CDPR and the robot arm are decoupled; (b) Integrated control-I with the CDPR and the robot arm are coupled; and (c) Integrated control-II with the CDPR and the robot arm are coupled.}
\label{fig:J1_ThreeCtrlStructure}
\end{figure}

\subsection{Independent Control}
\label{subsec:J1_IndependentCtrl}
For the independent plant (i.e., the Independent Arm \& CDPR block diagram shown in \autoref{fig:J1_ThreeCtrlStructure}(a)), it includes two independent subsystems: the flexible CDPR and the rigid robot arm. There is no coupled forces/torques between them. In this case, the dynamic model of the CDPR can be developed using \eqref{eq:J1_9} or by replacing \eqref{eq:J1_17} and \eqref{eq:J1_18} with ${K_E} = \frac{1}{2}{m_m}[{{\dot p}_{mx}},{{\dot p}_{my}},{{\dot p}_{mz}}]{[{{\dot p}_{mx}},{{\dot p}_{my}},{{\dot p}_{mz}}]^T} + \frac{1}{2}\omega _m^T{I_m}{\omega _m}$ and ${V_E} = {m_m}g{p_{mz}}+\frac{1}{2}{\left( {L - {L_0}} \right)^T}{K_c}\left( {L - {L_0}} \right)$, respectively. The dynamic equations of the robot arm are derived by replacing \eqref{eq:J1_17} and \eqref{eq:J1_18} with ${K_E} = \frac{1}{2}\sum\limits_{k = 0}^3 {\left\{ {{m_{ak}}v_{ack}^T{v_{ack}} + \omega _{ack}^T{I_{ak}}{\omega _{ack}}} \right\}}$ and ${V_E} = \sum\limits_{k = 0}^3 {\left\{ {{m_{ak}}g{p_{ack}^T}{{[0,0,1]}^T}} \right\}}$, respectively. By substituting the new \eqref{eq:J1_17} and \eqref{eq:J1_18} into \eqref{eq:J1_19}, the independent nonlinear dynamic equations of the CDPR and robot arm can be derived (in forms of \eqref{eq:J1_22} and \eqref{eq:J1_24}). The LTV model of the CDPR is expressed in form of \eqref{eq:J1_34}. The block diagram $\rm{S}$ is used to select elements from input vector. {\color{black}It can be described as
\begin{align} \label{eq:S_blockdiagram_a}
\left[\begin{array}{cc}
x_m \\
\hline
x_a \end{array}\right]
= \underbrace {\left[ \begin{array}{cc}
{{I_{6 \times 6}}}&{{{\bf{0}}_{6 \times 4}}}\\
\hline
{{{\bf{0}}_{4 \times 6}}}&{{I_{4 \times 4}}}
\end{array} \right]}_{\rm{Block~diagram~S}}
\left[\begin{array}{cc}
x \\
\hline
x \end{array}\right]
\end{align}
where $I$ represents the identity matrix. $x$, $x_m$, and $x_a$ denote the state vectors of the whole system, CDPR, and robot arm, respectively. In~\eqref{eq:S_blockdiagram_a}, $x$, $x_m$, and $x_a$ can be replaced with reference state vectors $x_r$, $x_{mr}$, and $x_{ar}$, respectively.}

\renewcommand{\algorithmicrequire}{\textbf{Input:}}  
\renewcommand{\algorithmicensure}{\textbf{Output:}} 
\begin{algorithm}[!t]
\caption{Computation of the optimal cable tensions $T_{3opt}$ and $T_{4opt}$}
\label{algorithm:J1_optimalRefInput}
\begin{algorithmic}[1]
\Require
    Reference trajectory $r(t)=x_r$.
\Ensure
    Optimal cable tensions $T_{3opt}$ and $T_{4opt}$.
\State Given $r(t)$ and calculate the nominal matrix $A_{mr}$ and torque $(A_mT)_r$ by using \eqref{eq:J1_5} and \eqref{eq:J1_32b}, respectively;
\State Given the stiffness weighting matrix ${H_\lambda}$ (e.g., in all the case studies in this paper, ${H_\lambda}$ is equal to the identity matrix $I$) and solve \eqref{eq:J1_32a} (which also subjects to \eqref{eq:J1_32c} and \eqref{eq:J1_32d}), then the optimal stiffness ${{K}(\lambda)}_{opt}$ and variable ${\lambda}_{opt}$ are obtained;
\State Resubstitute ${{K}(\lambda)}_{opt}$, ${\lambda}_{opt}$, $A_{mr}$, and $(A_mT)_r$ into \eqref{eq:J1_32c}, the optimal cable tensions $T_{3opt}$ and $T_{4opt}$ are computed;
\State {\bf Return} $T_{3opt}$ and $T_{4opt}$.
\end{algorithmic}
\end{algorithm}

\begin{figure}[!t]
\centering   
\subfigure[]{\label{fig:J1_JK_2dview}\includegraphics[height=32mm]{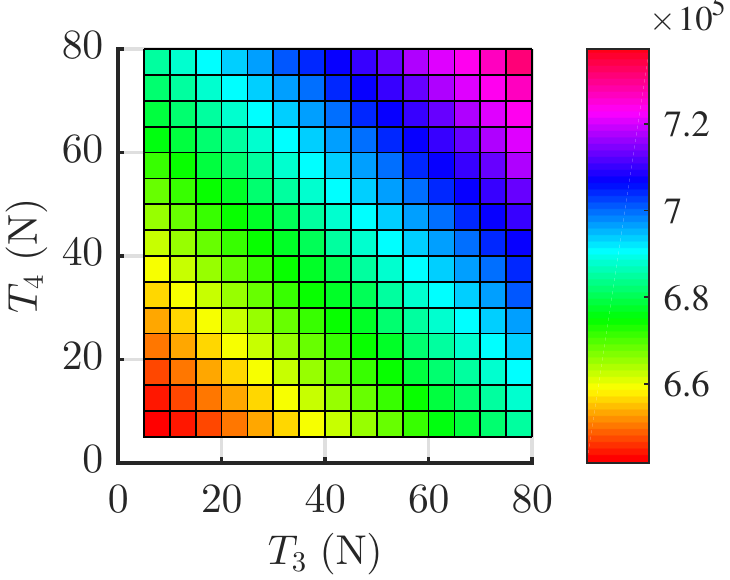}}
\vspace{-0.2cm}
\subfigure[]{\label{fig:J1_JK_3dview}\includegraphics[height=32mm]{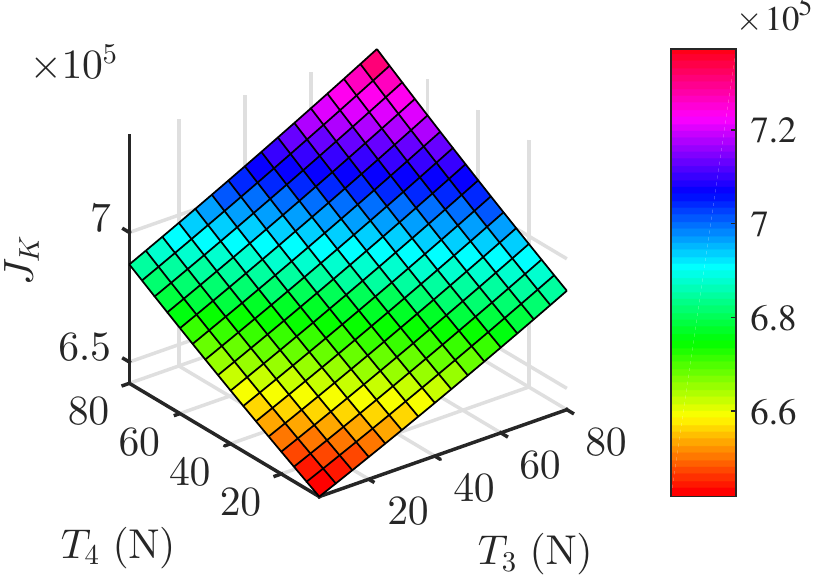}}
\vspace{-0.2cm}
\subfigure[]{\label{fig:J1_JK_eig}\includegraphics[width=85mm]{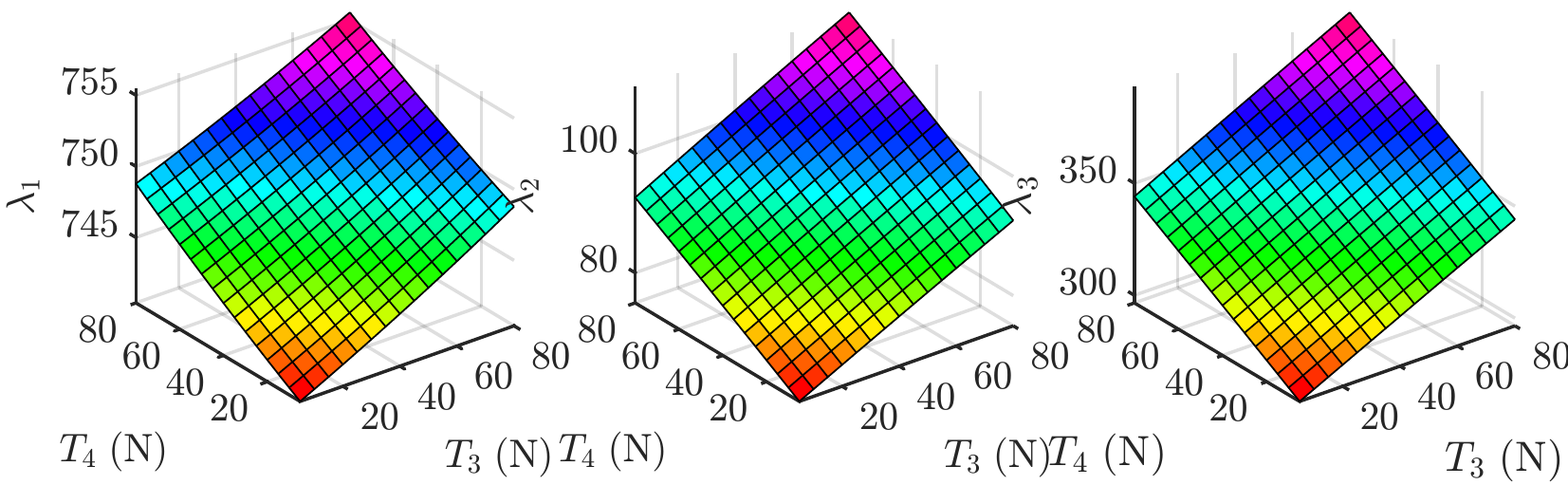}}
\vspace{-0.2cm}
\caption{An stiffness optimization example using{~}\autoref{algorithm:J1_optimalRefInput}, where the mobile platform is stationary (i.e., position-holding at $[p_{mx},p_{mz}]^T=[0,0]^T$) and the upper unstretched cable lengths are equal to $L_{01}=L_{02}=1.005 \; \rm{m}$. (a) $X$-$Y$ view, (b) 3D view, and (c) eigenvalues of the stiffness matrix $K$.}
\label{fig:J1_J_K}
\end{figure}

When the reference trajectory $r(t)=x_r$ is given (in \autoref{fig:J1_ThreeCtrlStructure}(a)),{~}\autoref{algorithm:J1_optimalRefInput} is implemented to compute the optimal cable tensions $T_{3opt}$ and $T_{4opt}$. By substituting $T_{3opt}$ and $T_{4opt}$ into \eqref{eq:J1_8}, the upper unstretched cable lengths $L_{01}$ and $L_{02}$ are also calculated. An stiffness optimization example using{~}\autoref{algorithm:J1_optimalRefInput} is shown in \autoref{fig:J1_J_K}. \autoref{fig:J1_J_K}(a) shows that increasing $T_3$ and $T_4$ will rise up $J_K$ (in \autoref{fig:J1_J_K}(b)). The corresponding eigenvalues of stiffness matrix $K$ (in \ref{fig:J1_J_K}(c)) are always positive. Because of the constraints in \eqref{eq:J1_32}, the maximum of $J_K$ corresponds to the optimal values of $T_3$ and $T_4$, i.e., $T_{3opt}$ and $T_{4opt}$, respectively.  

Based on the computed values above (in this section), the proposed control approaches are then utilized to stabilize the system around its reference trajectories. For the independent control, Model Predictive Control (MPC) and Proportional-Integral-Derivative (PID) based control schemes are designed. The former is used to control the lower cable tensions $u_m{({T_3},{T_4})}$, the latter is utilized to control the arm joint torques to minimize joint position errors. {\color{black}The main reason for using MPC-based controllers (see~\autoref{fig:J1_ThreeCtrlStructure}) is that they can handle optimal control problems as well as satisfying constraints.}
\begin{enumerate}[1)]
	\item MPC Control (depicted as the MPC control block diagram in \autoref{fig:J1_ThreeCtrlStructure}(a)): As an optimization based control approach, the MPC cost function with constraints is defined below to minimize the vibration of the mobile platform:
\begin{align}
\label{eq:J1_35}
&{\min} &&\sum\limits_{k = 0}^{{N_p} - 1} {\left( {{e_{x_m}}{{(k)}^T}{Q_p}{e_{x_m}}(k)} \right.} \left. { + e_{u_m}{{(k)}^T}{R_p}e_{u_m}(k)} \right) \nonumber \\
&&&+ {e_{x_m}}{({N_p})^T}{P_p}{e_{x_m}}({N_p})\\ 
&{\rm{s.\;t.}} && \delta {{{x_m}}}(k + 1)_{6 \times 1} = A(k)_{6 \times 6}\delta {x_m}(k)_{6 \times 1} \nonumber \\
&&&+ {B}(k)_{6 \times 2}\delta {u_m}(k)_{2 \times 1} + {B}(k)_{6 \times 2}w_m(k)_{2 \times 1}\nonumber \\
&&&\delta {y_m}(k)_{6 \times 1} = {C}(k)_{6 \times 6}\delta {x_m}(k)_{6 \times 1} \nonumber \\
&&&\delta x_m{(k)_{6 \times 1}} = x_m{(k)_{6 \times 1}} - x_m{(k - 1)_{6 \times 1}}\nonumber \\
&&&\delta u_m{(k)_{2 \times 1}} = u_m{(k)_{2 \times 1}} - u_m{(k - 1)_{2 \times 1}}\nonumber \\
&&&\delta y_m{(k)_{6 \times 1}} = y_m{(k)_{6 \times 1}} - y_m{(k - 1)_{6 \times 1}}\nonumber \\
&&&{e_{x_m}}(k)_{6 \times 1} = x_{mr}{(k)_{6 \times 1}} - x_m{(k)_{6 \times 1}}\nonumber \\
&&&{e_{u_m}}(k)_{2 \times 1} = u_{mr}{(k)_{2 \times 1}} - u_m{(k)_{2 \times 1}}\nonumber \\
&&&\delta {x_{mL}(k)}_{6 \times 1} \le \delta x_m(k)_{6 \times 1} \le \delta {x_{mU}(k)}_{6 \times 1}\nonumber \\
&&&\delta {u_{mL}(k)}_{2 \times 1} \le \delta u_m(k)_{2 \times 1} \le \delta {u_{mU}(k)}_{2 \times 1}\nonumber
\end{align}	
where the state-space model in \eqref{eq:J1_35} represents the independent CDPR. ${e_{x_m}} = x_{mr} - {x_m} = {[{p_{mxr}},{\dot p_{mxr}},{p_{mzr}},{\dot p_{mzr}},{\beta _{mr}},{\dot \beta _{mr}}]^T} - {[{p_{mx}},{\dot p_{mx}},{p_{mz}},{\dot p_{mz}},{\beta _m},{\dot \beta _m}]^T}$ are the errors between the reference trajectory $x_{mr}$ and actual states ${x_m}$. ${e_{u_m}} = u_{mr} - u_m = [{T_{3opt}},{T_{4opt}}]_{2 \times 1}^T - [{T_3},{T_4}]_{2 \times 1}^T$ denotes the errors between the reference inputs $u_{mr}$ and actual inputs $u_m$. 
$\delta x_{mL}$ and $\delta x_{mU}$ denote the lower bound and upper bound of the states $\delta x_m$, respectively. $\delta {u_L}$ and $\delta {u_U}$ represent the lower bound and upper bound of the control inputs $\delta u$, respectively. ${R_p} \in {\mathbb{R}^{2 \times 2}}$ $({R_p} = R_p^T \succ \textbf{0})$, ${Q_p} \in {\mathbb{R}^{6 \times 6}}$ $({Q_p} = Q_p^T \succeq \textbf{0})$, and ${P_p} \in {\mathbb{R}^{6 \times 6}}$ $({P_p} = P_p^T \succeq \textbf{0})$ are input, state, and terminal weighting matrices, respectively. 

{\color{black}
{\quad}As an advanced control technique, MPC utilizes the state-space model in~\eqref{eq:J1_35} to predict the system’s behavior in the future, i.e., by minimizing the cost function (minimizing vibrations and tracking errors in joint space) as well as handling constraints to find the optimal control action. To carry out the MPC in~\eqref{eq:J1_35}, parameters (see~\autoref{table:J1_ThreeControlStructures} in~\autoref{subsec:J1_CaseStudyComparisonControlStructures} for case studies) such as sampling time $T_s$, predictive horizon $N_p$, control horizon $N_c$, input weighting matrix $R_p$, state weighting matrix $Q_p$, terminal weighting matrix $P_p$, lower bound $\delta {x_{mL}}$, upper bound $\delta {x_{mU}}$, lower bound $\delta {u_{mL}}$, and upper bound $\delta {u_{mU}}$ must be considered. Additionally, some key guidelines can help tuning MPC parameters such as set smaller $N_p$ and $N_c \le N_p$.}

\item PID Control (shown in \autoref{fig:J1_ThreeCtrlStructure}(a)): For the mounted arm, the PID controller is designed as
\begin{align}
&{u_a}{({\tau _{a2}},{\tau _{a3}})} =
{K_p}{[{\theta _{a2r}}(t) - {\theta _{a2}}(t),{\theta _{a3r}}(t) - {\theta _{a3}}(t)]^T}\nonumber \\
&\qquad + {K_i}\int_0^t {{{[{\theta _{a2r}}(\mathfrak{t}) - {\theta _{a2}}(\mathfrak{t}),{\theta _{a3r}}(\mathfrak{t}) - {\theta _{a3}}(\mathfrak{t})]}^T}d\mathfrak{t}} \nonumber \\
&\qquad + {K_d}{[{{\dot \theta }_{a2r}}(t) - {{\dot \theta }_{a2}}(t),{{\dot \theta }_{a3r}}(t) - {{\dot \theta }_{a3}}(t)]^T}
\label{eq:J1_36}
\end{align}
where $K_p$, $K_i$, and $K_d$ are the proportional, integral, and derivative terms, respectively. ${\theta _{a2r}}$ and ${\theta _{a3r}}$ denote the reference angles of joint 2 and 3, respectively. ${\theta _{a2}}$ and ${\theta _{a3}}$ represent the actual angles of joint 2 and 3, respectively. ${u_a}{({\tau _{a2}},{\tau _{a3}})_{2 \times 1}}$ denotes the control input to the robot arm. {\color{black}The PID controllers (see~\autoref{fig:J1_ThreeCtrlStructure}) provide a model-free control strategy to compute the error dynamics for control performance comparison.}
\end{enumerate}

\subsection{Integrated Control-I}
\label{subsec:J1_IntegratedCtrlI}
Regarding the integrated control-I, it is also based on hybrid MPC and PID controllers (shown in \autoref{fig:J1_ThreeCtrlStructure}(b)). In this case, a coupled plant (the HCDR block diagram in \autoref{fig:J1_ThreeCtrlStructure}(b)) is adopted, i.e., the nonlinear model \eqref{eq:J1_33}. The corresponding LTV representation of the CDPR is obtained by linearizing \eqref{eq:J1_33} around the reference states $x_{mr}$ and inputs $u_{mr}$, which is used for MPC design (in the form of \eqref{eq:J1_35}) to damp vibrations. Meanwhile, the PID control design for the robot arm is expressed as \eqref{eq:J1_36}.

Additionally, the implementation of integrated control-I is the
same as the independent control (in \autoref{subsec:J1_IndependentCtrl}): when the reference trajectory $r(t)$ is given, the optimal cable tensions $T_{3opt}$ and $T_{4opt}$ and the upper unstretched cable lengths $L_{01}$ and $L_{02}$ are computed. By inputting these values and the outputs of MPC and PID into HCDR to minimize vibration and improve the accuracy of the end-effector.

\subsection{Integrated Control-II (Fully Integrated Control)}
\label{subsec:J1_IntegratedCtrlII}
Integrated control-II is defined as a fully integrated control, which is only based on MPC (shown in \autoref{fig:J1_ThreeCtrlStructure}(c)). In this case, the coupled plant (the HCDR block diagram in \autoref{fig:J1_ThreeCtrlStructure}(c)) is the same as the one shown in \autoref{fig:J1_ThreeCtrlStructure}(b), but the corresponding LTV model is extended to all the states and inputs, i.e., in terms of \eqref{eq:J1_34}.

The integrated control-II is designed to control lower cable tensions and the arm joint torques simultaneously to minimize the vibration of the overall system. Then, the MPC cost function with constraints is redefined as
\begin{align}
\label{eq:J1_37}
&{\min} &&\sum\limits_{k = 0}^{{N_p} - 1} {\left( {{e_x}{{(k)}^T}{Q_p}{e_x}(k)} \right.} \left. { + e_u{{(k)}^T}{R_p}e_u(k)} \right) \nonumber \\
&&&+ {e_x}{({N_p})^T}{P_p}{e_x}({N_p})\\ 
&{\rm{s.\;t.}} && \delta {{{x}}}(k + 1)_{10 \times 1} = A(k)_{10 \times 10}\delta {x}(k)_{10 \times 1} \nonumber \\
&&&+ {B}(k)_{10 \times 4}\delta {u}(k)_{4 \times 1} + {B}(k)_{10 \times 4}w(k)_{4 \times 1}\nonumber \\
&&&\delta {y}(k)_{10 \times 1} = {C}(k)_{10 \times 10}\delta {x}(k)_{10 \times 1} \nonumber \\
&&&\delta x{(k)_{10 \times 1}} = x{(k)_{10 \times 1}} - x{(k - 1)_{10 \times 1}}\nonumber \\
&&&\delta u{(k)_{4 \times 1}} = u{(k)_{4 \times 1}} - u{(k - 1)_{4 \times 1}}\nonumber \\
&&&\delta y{(k)_{10 \times 1}} = y{(k)_{10 \times 1}} - y{(k - 1)_{10 \times 1}}\nonumber \\
&&&{e_x}(k)_{10 \times 1} = x_r{(k)_{10 \times 1}} - x{(k)_{10 \times 1}}\nonumber \\
&&&{e_u}(k)_{4 \times 1} = u_r{(k)_{4 \times 1}} - u{(k)_{4 \times 1}}\nonumber \\
&&&\delta {x_L(k)}_{10 \times 1} \le \delta x(k)_{10 \times 1} \le \delta {x_U(k)}_{10 \times 1}\nonumber \\
&&&\delta {u_L(k)}_{4 \times 1} \le \delta u(k)_{4 \times 1} \le \delta {u_U(k)}_{4 \times 1}\nonumber
\end{align}
where the errors between the reference trajectory $x_r$ and actual states $x$ are described as {\color{black}${e_x} = {x_r} - {x} = {[{p_{mxr}},{\dot p_{mxr}},{p_{mzr}},{\dot p_{mzr}},{\beta _{mr}},{\dot \beta _{mr}},{\theta _{a2r}},{{\dot \theta }_{a2r}},{\theta _{a3r}},{{\dot \theta }_{a3r}}]_{10 \times 1}^T}\\ - {[{p_{mx}},{\dot p_{mx}},{p_{mz}},{\dot p_{mz}},{\beta _m},{\dot \beta _m},{\theta _{a2}},{{\dot \theta }_{a2}},{\theta _{a3}},{{\dot \theta }_{a3}}]_{10 \times 1}^T}$,} and the errors between the reference inputs $u_r$ and actual inputs $u$ are expressed as ${e_u} = u_r - u = [{T_{3opt}},{T_{4opt}},{\tau _{a2r}},{\tau _{a3r}}]_{4 \times 1}^T - [{T_3},{T_4},{\tau _{a2}},{\tau _{a3}}]_{4 \times 1}^T$. Compared with \eqref{eq:J1_35}, other variables (e.g., $\delta x_L$, $\delta x_U$, $\delta {u_L}$, $\delta {u_U}$, ${R_p}$, ${Q_p}$, and ${P_p}$) are extended to higher dimensions. ${R_p} \in {\mathbb{R}^{4 \times 4}}$ $({R_p} = R_p^T \succ \textbf{0})$, ${Q_p} \in {\mathbb{R}^{10 \times 10}}$ $({Q_p} = Q_p^T \succeq \textbf{0})$, and ${P_p} \in {\mathbb{R}^{10 \times 10}}$ $({P_p} = P_p^T \succeq \textbf{0})$.

Moreover, when the reference trajectory $r(t)$ is given, the nominal variables $T_{3opt}$, $T_{4opt}$, $L_{01}$, $L_{02}$, ${\tau _{a2r}}$, and ${\tau _{a3r}}$ are computed the same as integrated control-I (in \autoref{subsec:J1_IntegratedCtrlI}). Theoretically, when the the same goal and conditions are given, the higher integrated control techniques (e.g., the integrated control-II in \autoref{subsec:J1_IntegratedCtrlII}) are easier lead to better performance, since the control performance indices are more guaranteed by balancing control gains, e.g., using the cost function in \eqref{eq:J1_37}.

{\color{black}In summary, the independent controller provides a model-free control law for the decoupled HCDR; the integrated control-I offers a combination of model-free and model-based control laws for the coupled HCDR. In contrast, the integrated control-II gives a fully model-based control law for the coupled HCDR.} Additionally, the block diagrams (in{~}\autoref{fig:J1_ThreeCtrlStructure}) of \textit{Inverse Dynamics}, \textit{Redundant Resolution {\&} Stiffness Optimization}, \textit{Independent Arm {\&} CDPR} (or \textit{HCDR}), \textit{PID}, and \textit{MPC} mainly correspond to \eqref{eq:J1_32b},{~}\autoref{algorithm:J1_optimalRefInput}, \eqref{eq:J1_33},{~}\eqref{eq:J1_36}, and \eqref{eq:J1_35} (or \eqref{eq:J1_37}), respectively. In the next section, case studies will be proposed to evaluate the control performance.

\section{Control Performance and Evaluation}
\label{sec:J1_ControlPerformanceandEva}
\subsection{Case Study--Comparison of Three Control Structures}
\label{subsec:J1_CaseStudyComparisonControlStructures}
To evaluate the performances of the above three control strategies, many case studies can be implemented, e.g., applying different trajectories to the mobile platform and robot arm. However, when the robot arm moves, it generates reaction forces which result in vibration of the mobile platform even when the desired position of the mobile platform is to remain unchanged. This case is quite important in pick-and-place applications. To illustrate the position-holding performance of the CDPR and the position accuracy performance of the end-effector relative to its reference trajectory, reference points $r(t) = {x_r} = {[{p_{mxr}},{\dot p_{mxr}},{p_{mzr}},{\dot p_{mzr}},{\beta _{mr}},{\dot \beta _{mr}},{\theta _{a2r}},{{\dot \theta }_{a2r}},{\theta _{a3r}},{{\dot \theta }_{a3r}}]^T}$ are given as $\underset{t_0}{r_0}=\underset{t_A}{r_A}=[0.05,0,0.1,0,0,0,0,0,0,0]^T $ $\rightarrow$ 
$\underset{t_B}{r_B}=[0.05,0,0.1,0,0,0,0,0,0.3(t_B-t_A),0]^T $ $\rightarrow$
$\underset{t_C}{r_C}=[0.05,0,0.1,0,0,0,0.4(t_C-t_B),0,0.3(t_C-t_B),0]^T $ $\rightarrow$
$\underset{t_{end}}{r_{end}}=[0.05,0,0.1,0,0,0,1.0(t_D-t_C),0,1.0(t_D-t_C),0]^T $, where point-to-point (e.g., $r_B \rightarrow r_C$ from time $t_B$ to  $t_C$) movements are implemented using the 5th order polynomial trajectories, and $t_0=0\;\rm{s}$, $t_A=1\;\rm{s}$,  $t_B=3\;\rm{s}$, $t_C=5\;\rm{s}$, and $t_{end}=6\;\rm{s}$. {\color{black}Since the Cartesian position $(x_e,0,z_e)=p_e$ of the end-effector is expressed in terms of $r(t)$ (using the equations shown in Appendix \ref{appendix:J1_9DOFDynamics}), then vibrations and tracking errors of the end-effector can be evaluated though different controllers in \autoref{sec:J1_VibCtrlDesign}.} The corresponding multi segment curves are generated: from the start point $\rightarrow$ point A $\rightarrow$ point B $\rightarrow$ point C $\rightarrow$ the end point as shown in \autoref{fig:J1_EndEffectorCartTraj}.

Furthermore, the control performance was evaluated using MATLAB 2015a (The MathWorks, Inc.) on a Windows 7 x64 desktop PC (Intel Core i7-4770, 3.4 GHz CPU and 8 GB RAM), and the quadratic programming problems (\eqref{eq:J1_35} and \eqref{eq:J1_37}) in the independent control, integrated control-I, and integrated control-II were solved using FiOrdOs \cite{F.Ullmann2011}. The constraints and tuning parameters are given in \autoref{table:J1_ThreeControlStructures}.

\begin{table}[h]
	\renewcommand{\arraystretch}{1.3}
	\caption{Parameters of Three Control Structures}
	\centering
	\label{table:J1_ThreeControlStructures}
	\resizebox{\columnwidth}{!}{
		\begin{tabular}{c c c}
			\hline\hline \\[-3mm]
			\multicolumn{1}{c}{Control structures} & \multicolumn{1}{c}{MPC parameters
} & \multicolumn{1}{c}{\pbox{20cm}{PID parameters}}  \\[1.6ex] \hline
\pbox{20cm}{Independent control \\ or \\ Integrated control-I}	
& \pbox{20cm}{$T_s=0.01\;\rm{s}$ (sampling time); \\$N_p=50$ (predictive horizon); \\$N_c=50$ (control horizon); \\$R_p=0.0001I_{2 \times 2}$ (input weighting matrix); \\
$Q_p=I_{4 \times 4}$ (state weighting matrix); \\
$P_p=I_{4 \times 4}$ (terminal weighting matrix); \\  
$\delta {x_{mL}}=-[\infty,\infty]^T$(lower bound); \\ 
$\delta {x_{mU}}=[\infty,\infty]^T$(upper bound);\\
$\delta {u_{mL}}=-[80,80]^T$(lower bound); \\ 
$\delta {u_{mU}}=[80,80]^T$(upper bound).} & \pbox{20cm}{$K_p=400$;\\ $K_i=100$; \\$K_d=10$.} \\
\pbox{20cm}{Integrated control-II}	
& \pbox{20cm}{$T_s=0.01\;\rm{s}$; $N_p=50$;
$N_c=50$; \\ $R_p=0.0001I_{4 \times 4}$;
$ Q_p=I_{10 \times 10}$; $P_p=I_{10 \times 10}$;\\ 
$\delta {x_L}=-[\infty,\infty,\infty,\infty]^T$;
$\delta {x_U}=[\infty,\infty,\infty,\infty]^T$;\\
$\delta {u_L}=-[80,80,2,2]^T$; 
$\delta {u_U}=[80,80,2,2]^T$.} & \pbox{20cm}{--}\\ [1.4ex]
\hline\hline
\end{tabular}
}
\end{table}

Based on the desired end-effector trajectory and tuning parameters of three control structures, the performance of the proposed control systems is shown in \autoref{fig:J1_EndEffectorCartTraj} and \autoref{fig:J1_EndEffectorTrajVSTime}. \autoref{fig:J1_EndEffectorCartTraj} shows the end-effector trajectory in Cartesian coordinates. The aim of three controllers is to follow the desired curved path (dotted line). The independent control based tracking trajectory (dashed line), the integrated control-I based tracking trajectory (dash-dot line), and the integrated control-II based tracking trajectory (solid line) are all commanded from the same start point. It is clear that the independent control cannot follow the desired path well. The main reason is that it doesn\textrm{'}t consider the coupling forces/torques between the mobile platform and the robot arm. This leads to large tracking errors.
\begin{figure}[t]\centering
	\includegraphics[width=7.2cm]{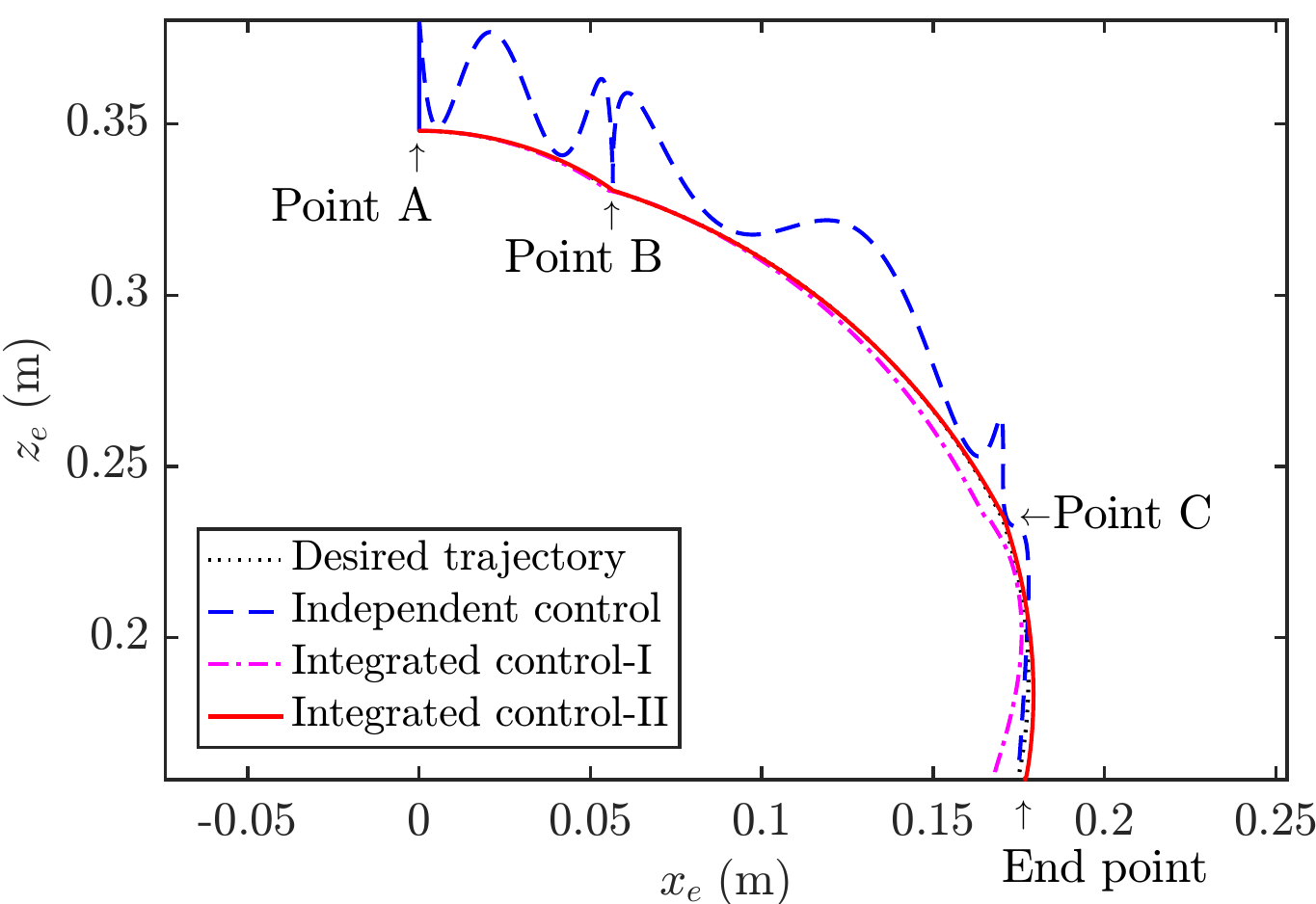}
	\vspace{-0.2cm}
	\caption{End-effector trajectory in Cartesian coordinates.}
	\label{fig:J1_EndEffectorCartTraj}
	\vspace{-0.2cm}
\end{figure}
\vspace{-0.1cm}
\begin{figure}[h]\centering
	\includegraphics[width=8.5cm]{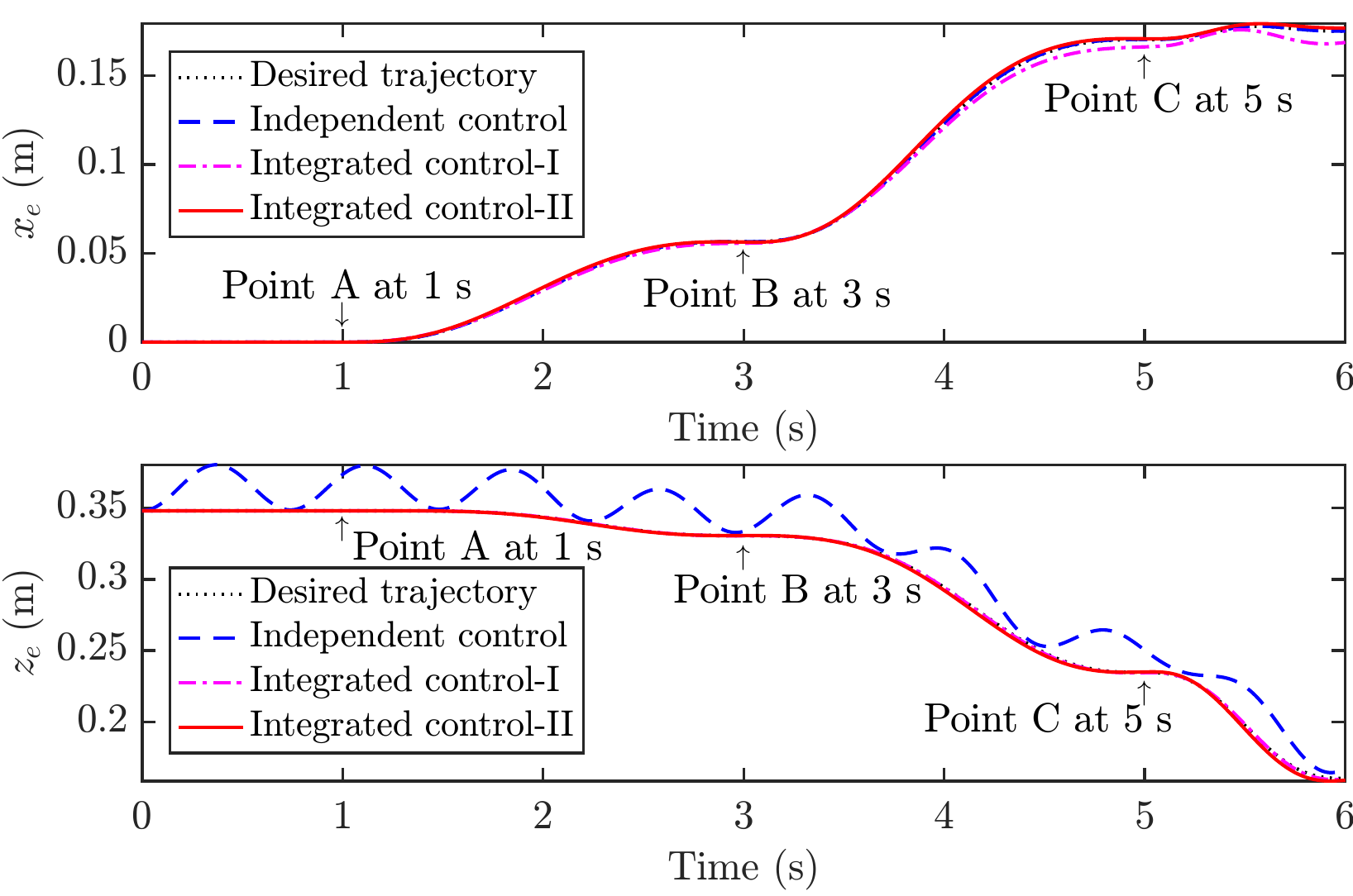}
	\vspace{-0.3cm}
	\caption{End-effector trajectory versus time.}		\label{fig:J1_EndEffectorTrajVSTime}
	\vspace{-0.3cm}
\end{figure}

Additionally, from the start point $\rightarrow$ point A $\rightarrow$ point B, integrated control-I and integrated control-II show good trajectory tracking performance. However, from point B $\rightarrow$ point C $\rightarrow$ end point, tracking errors of former are larger than the later (always has a good tracking performance). Integrated control-II uses an optimized control scheme to handle dynamic coupling hence suppressing vibrations to satisfy \eqref{eq:J1_37}.

\autoref{fig:J1_EndEffectorTrajVSTime} shows the end-effector trajectory versus time. The time response for three control structures has a similar tracking performance, as shown in \autoref{fig:J1_EndEffectorCartTraj}. In \autoref{fig:J1_EndEffectorTrajVSTime}, both the independent control (for the decoupled HCDR) and the integrated control-I (for the coupled HCDR) use the same controllers (MPC and PID) and tuning parameters (shown in \autoref{table:J1_ThreeControlStructures}). The integrated control-I shows good tracking performances in the $X$- and $Z$-directions, and the control inputs can handle the reaction forces between the CDPR and robot arm. However, for the independent control, the control inputs cannot effectively handle the decoupled HCDR (i.e., the ignored reaction force mainly coming from the gravity of the robot arm cannot be overcome) in the $Z$-direction, resulting in vibrations and poor tracking. The reaction force in the $X$-direction is less affected, and the tuned controllers can effectively eliminate the tracking error, so it shows a good tracking performance.

In short, the above results show that integrated control-II (fully integrated control) has better tracking performance than that of integrated control-I. Integrated control-I has better tracking performance than that of independent control.

\subsection{Case Study--RMSE Estimation}
To evaluate the end-effector position tracking errors, the root-mean-square error (RMSE) \cite{F.Y.Wu2015} is used to measure the differences between the desired positions in $X$-$Z$ Cartesian plane and observed values. The RMSE of the end-effector trajectory is described as
\begin{align}
RMSE = \sqrt {\frac{1}{{{N_R}}}\sum\limits_{i = 1}^{{N_R}} {({{(p_{exi}-\hat p_{exi})}^2} + {{(p_{ezi}-\hat p_{ezi})}^2})} }
\label{eq:J1_38}
\end{align}
where $p_{exi}$, $\hat p_{exi}$, $p_{ezi}$, and $\hat p_{ezi}$ denote the desired and observed end-effector positions in $X$- and $Z$-directions, respectively. $N_R$ is the total sampling number.

Using \eqref{eq:J1_38}, the RMSEs of the end-effector trajectory based on the independent control, integrated control-I, and integrated control-II are shown in \autoref{fig:J1_EndEffectorRMSE}. In \autoref{fig:J1_EndEffectorRMSE}, RMSE in the $X$-direction, independent control has the smallest RMSE, i.e., the best trajectory tracking performance. RMSE in the $Z$-direction, integrated control-I has the best trajectory tracking performance. However, RMSE in the 2D-direction represents the overall trajectory tracking performance of the three control structures. It is clear that integrated control-II has the best trajectory tracking performance ($RMSE=0.01889$), and independent control has the worst trajectory tracking performance ($RMSE=0.00164$). Also, this performance matches the result shown in \autoref{fig:J1_EndEffectorCartTraj} and \autoref{fig:J1_EndEffectorTrajVSTime}. Hence, integrated control-II (fully integrated control) has the best overall trajectory tracking performance for the end-effector.

To the best of the authors\textrm{'} knowledge, few studies use MPC, which utilizes a sufficiently accurate dynamic model. In comparison with previous studies, such as PID \cite{Mendez2014}, linear parameter-varying (LPV) \cite{H.Jamshidifar2018}, and sliding mode control (SMC) \cite{R.d.Rijk2018}, the results of this paper offer noticeable improvements in the following aspects: 1) satisfactory results are guaranteed by the optimal control inputs and constraints, and 2) the use of MPC enhances the control performance by using the future steps from the reference trajectories to generate control laws.

{\color{black}Additionally, various simulation environments can be used to validate the control performance above, such as Robot Operating System (ROS) and Matlab. In this paper, we use Matlab because it can quickly realize code generation for experiments (e.g., implemented in a Beckhoff Embedded PC) in the future.}

\begin{figure}[t]\centering
	\includegraphics[width=8.7cm]{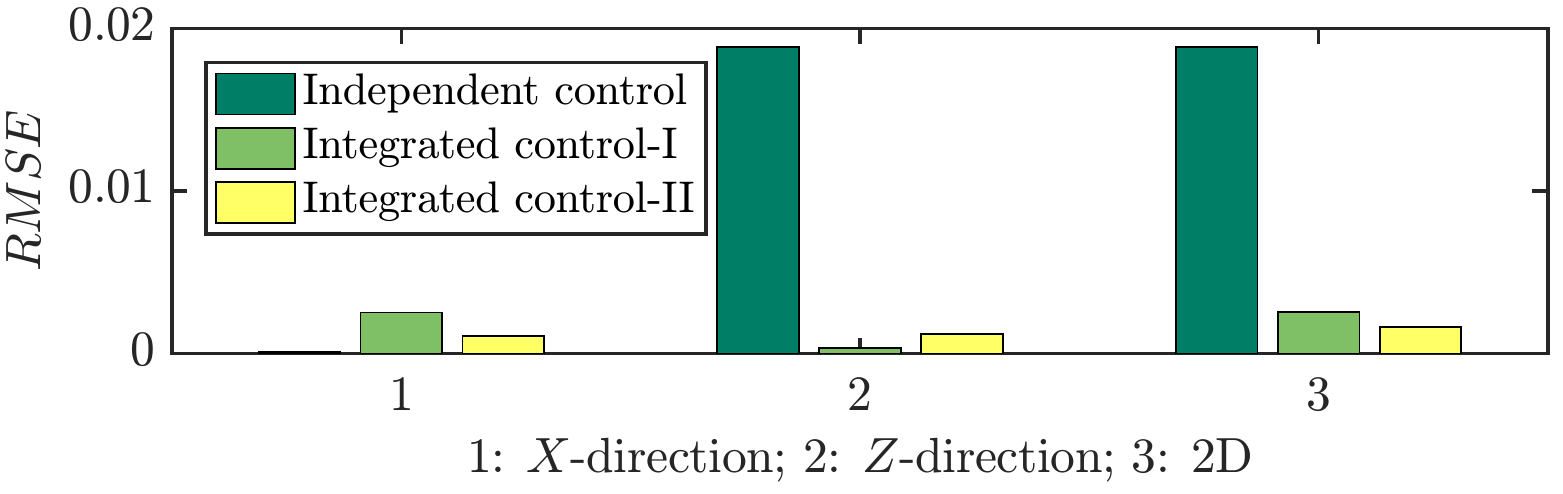}
	\caption{RMSE of the end-effector trajectory.}		\label{fig:J1_EndEffectorRMSE}
\end{figure}

\section{Conclusions and Future Work}
\label{sec:J1_Conclusions}
This paper presented a generalized HCDR by combining the strengths and benefits of serial and cable-driven-parallel robots. A generalized modeling approach was also proposed for the HCDR, including equations of motion, redundancy resolution, and stiffness optimization. This approach can be extended to other types of hybrid robots.

In addition, three control architectures were developed and analyzed to achieve the goal of reducing vibrations and trajectory tracking errors for the end-effector. Control performances in different aspects, including the position-holding performance of the CDPR and the position accuracy performance of the end-effector were evaluated and discussed. The results showed that the fully integrated control system could reduce the tracking and end-effector vibrations significantly.

In the future, the hardware of HCDR is planned to be designed, and then experiments for the proposed control strategies will be conducted. A common performance indicator that can easily compare the proposed control strategies with other types of controllers will also be developed. {\color{black}\autoref{fig:J1_ThreeCtrlStructure} provided three control architectures (input white noises were also considered), and they were evaluated in simulations. In the future, some challenging problems in real-world applications such as input dead zones~\cite{TYang2019}, parametric uncertainties~\cite{NSun2019}, and unidirectional input constraints~\cite{NSun2019} are planned to be studied as well.}

\appendices
\section{} \label{appendix:J1_9DOFDynamics}
For the specific 9-DOF HCDR, the COM (of the links) positions are computed as ${p_{ac1}} = {p_{a0}} + R_g^m{R_z}({\theta _{a1}}){[x_{ac01},y_{ac01},z_{ac01}]^T}$, ${p_{ac2}} = {p_{a1}} + R_g^m{R_z}({\theta _{a1}}){R_y}({\theta _{a2}}){[x_{ac02},y_{ac02},z_{ac02}]^T}$, and ${p_{ac3}} = {p_{a2}} + R_g^m{R_z}({\theta _{a1}}){R_y}({\theta _{a2}}){R_y}({\theta _{a3}}){[x_{ac03},y_{ac03},z_{ac03}]^T}$ where the joint position vectors are described as ${p_{a0}} = {[{p_{mx}},{p_{my}},{p_{mz}}]^T} + R_g^m{l_m}$, ${p_{a1}} = {p_{a0}} + R_g^m{R_z}({\theta _{a1}}){[x_{a01},y_{a01},z_{a01}]^T} $, ${p_{a2}} = {p_{a1}} + R_g^m{R_z}({\theta _{a1}}){R_y}({\theta _{a2}}){[x_{a02},y_{a02},z_{a02}]^T}$, and ${p_e}={p_{a3}} = {p_{a2}} + R_g^m{R_z}({\theta _{a1}}){R_y}({\theta _{a2}}){R_y}({\theta _{a3}}){[x_{a03},y_{a03},z_{a03}]^T}$. Additionally, the COM linear velocities and angle velocities (of the links) are calculated as ${v_{ac1}} = {{\dot p}_{ac1}}$, ${v_{ac2}} = {{\dot p}_{ac2}}$, ${v_{ac3}} = {{\dot p}_{ac3}}$, ${\omega _{ac1}} = {({R_z}({\theta _{a1}}))^T}{\omega _m} + {[0,0,{{\dot \theta }_{a1}}]^T}$, ${\omega _{ac2}} = {({R_z}({\theta _{a1}}){R_y}({\theta _{a2}}))^T}({\omega _m} + {[0,0,{{\dot \theta }_{a1}}]^T}) + {[0,{{\dot \theta }_{a2}},0]^T}$, and ${\omega _{ac3}} = {({R_z}({\theta _{a1}}){R_y}({\theta _{a2}}){R_y}({\theta _{a3}}))^T}({\omega _m} + {[0,0,{{\dot \theta }_{a1}}]^T}) + {({R_y}({\theta _{a2}}){R_y}({\theta _{a3}}))^T}{[0,{{\dot \theta }_{a2}},0]^T} + {[0,{{\dot \theta }_{a3}},0]^T}$, where the corresponding parameters are shown in \autoref{fig:J1_9dofHCDPR} and \autoref{table:J1_HCDPRParameters}. By substituting these corresponding equations into \eqref{eq:J1_16} and \eqref{eq:J1_17}, respectively, the equations of motion of the 9-DOF HCDR can be derived (in forms of \eqref{eq:J1_22} and \eqref{eq:J1_24}). {\color{black}One can also verify the equations of motion (symbolic formulas) above using commercial software such as MapleSim.}

\section{} \label{appendix:J1_DroneArm_Dynamics}
{\color{black}
\begin{figure}[h]\centering
	\includegraphics[width=8.5cm]{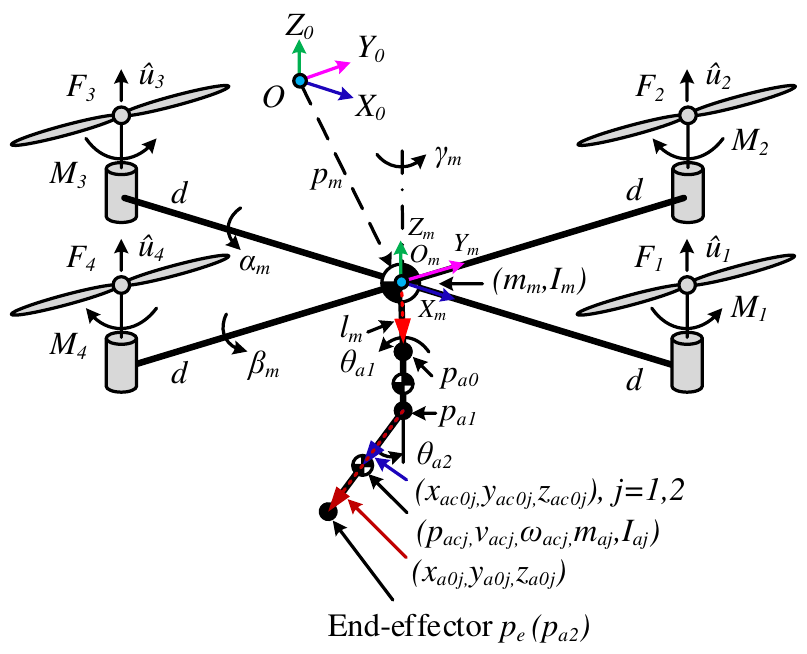}
	\caption{The schematic with frames assignment of a quadrotor and a 2-DOF mounted upside-down robot arm and forces/moments acting on the quadrotor.}\label{fig:J1_appx_quadrotor}
\end{figure}
To verify the proposed dynamic modeling approach (see~\autoref{sec:J1_GeneralizedSystemModeling}) that can be easily extended to other types of hybrid robots. Here, we take a configuration of Hummingbird Quadrotor~\cite{Mellinger2012} as an example for analysis; meanwhile, a 2-DOF rotational robot arm (the first and second revolute joints rotate around their body-fixed $Z$- and $Y$-axes, respectively) is mounted upside down. The schematic with frames assignment and forces/moments acting on the quadrotor are shown in~\autoref{fig:J1_appx_quadrotor}. We choose $Z-X-Y$ order to obtain the rotation matrix $R_g^m$, i.e., $R_g^m = R_{z}({\gamma _m})R_x({\alpha _m})R_{y}({\beta _m})$, where ${\alpha _m}$, ${\beta _m}$, and ${\gamma _m}$ are Euler angles rotating around the $X$-, $Y$-, and $Z$-axes, respectively. $F_i$ and $M_i$ ($i=1,2,3,4$) represent a force and moment produced by each rotor. To begin with, we can computer the equations of motion of the quadrotor (without considering the robot arm) as follows: ${{\vec F}_1} + {{\vec F}_2} + {{\vec F}_3} + {{\vec F}_4} - {[0,0,{m_m}g]^T} = {m_m}{{\dot v}_m}$ and $R_g^m[r_{1x},r_{1y},r_{1z}]^T \times {{\vec F}_1} + R_g^m[r_{2x},r_{2y},r_{2z}]^T \times {{\vec F}_2} + R_g^m[r_{3x},r_{3y},r_{3z}]^T \times {{\vec F}_3} + R_g^m[r_{4x},r_{4y},r_{4z}]^T \times {{\vec F}_4} + {M_1}{{\hat u}_1} - {M_2}{{\hat u}_2} + {M_3}{{\hat u}_3} - {M_4}{{\hat u}_4} = ({\color{black}R_g^m}{{{I}}_m{\color{black}R_g^m}^T})({\color{black}R_g^m}{{{{\dot \omega }}}_m}) + ({\color{black}R_g^m}{{{\omega }}_m}) \times (({\color{black}R_g^m}{{{I}}_m{\color{black}R_g^m}^T})({\color{black}R_g^m}{{{\omega }}_m}))$, where ${{\vec F}_i} = {F_i}{{\hat u}_i}$ and unit vector ${{\hat u}_i} = R_g^m{[0,0,1]^T}$. The body-fixed positions of four rotors $[r_{1x},r_{1y},r_{1z}]^T$, $[r_{2x},r_{2y},r_{2z}]^T$, $[r_{3x},r_{3y},r_{3z}]^T$, and $[r_{4x},r_{4y},r_{4z}]^T$ are equal to ${[d,0,0]^T}$, ${[0,d,0]^T}$, ${[ - d,0,0]^T}$, and ${[0, - d,0]^T}$, respectively. ${{{v}}_m}$, ${{{\omega}}_m}$, ${{\dot v}_m}$, and ${{\dot \omega}_m}$ are linear velocity, angular velocity, linear acceleration, and angular acceleration, respectively. Other parameters are shown in~\autoref{fig:J1_appx_quadrotor}. Then, the above equations can be expressed in terms of~\eqref{eq:J1_7} ($F_e$ and $M_e$ are not considered), in which 
\begin{align}
&T:=[F_1, F_2, F_3, F_4, M_1, M_2, M_3, M_4]^T, \label{eq:J1_Appx_B_quadrotor_T}\\
&A_m:=\nonumber\\
&{\begin{bmatrix}
{{\hat u}_1} & \cdots & {{\hat u}_4} & \bf{0}& \cdots & \bf{0}\\
{R_g^m\begin{bmatrix} r_{1x} \\ r_{1y} \\ r_{1z} \end{bmatrix}} \times {{\hat u}_1} & \cdots & {R_g^m\begin{bmatrix} r_{4x} \\ r_{4y} \\ r_{4z} \end{bmatrix}} \times {{\hat u}_4} & {{\hat u}_1} & \cdots & -{{\hat u}_4}
\end{bmatrix}} \label{eq:J1_Appx_B_quadrotor_Am}
\end{align}
\begin{align}
&=
\left[\begin{matrix}
{R_g^m}{\begin{bmatrix} 0 \\ 0 \\ 1 \\ \end{bmatrix}} & {R_g^m}{\begin{bmatrix} 0 \\ 0 \\ 1 \\ \end{bmatrix}} &
{R_g^m}{\begin{bmatrix} 0 \\ 0 \\ 1 \\ \end{bmatrix}} & {R_g^m}{\begin{bmatrix} 0 \\ 0 \\ 1 \\ \end{bmatrix}} \\
{R_g^m}{\begin{bmatrix} 0 \\ -d \\ 0 \\ \end{bmatrix}} & {R_g^m}{\begin{bmatrix} d \\ 0 \\ 0 \\ \end{bmatrix}} &
{R_g^m}{\begin{bmatrix} 0 \\ d \\ 0 \\ \end{bmatrix}}&{R_g^m}{\begin{bmatrix} -d \\ 0 \\ 0 \\ \end{bmatrix}} 
\end{matrix}\right.\nonumber\\
&\qquad\quad
\left.\begin{matrix}
\bf{0}& \bf{0}& \bf{0}& \bf{0}\\
{R_g^m}{\begin{bmatrix} 0 \\ 0 \\ 1 \\ \end{bmatrix}} & -{R_g^m}{\begin{bmatrix} 0 \\ 0 \\ 1 \\ \end{bmatrix}} &
{R_g^m}{\begin{bmatrix} 0 \\ 0 \\ 1 \\ \end{bmatrix}} & -{R_g^m}{\begin{bmatrix} 0 \\ 0 \\ 1 \\ \end{bmatrix}}
\end{matrix}\right], \label{eq:J1_Appx_B_quadrotor_Am2}
\end{align}
and $A_m$ represents the structure matrix of the quadrotor. $A_m$~\eqref{eq:J1_Appx_B_quadrotor_Am} can also be extended to other multirotors, such as hexarotor ($i=6$) and octorotor ($i=8$).

Since there is a mapping between $M_i$ and $F_i$: $M_i=k_MF_i/k_F$, where $k_M$ and $k_F$ are constants~\cite{Mellinger2012}. Then, the equations of motion of the quadrotor can be simplified as
\begin{align}
&{\begin{bmatrix}
{{m_m}{{{{\dot v}}}_m}}\\
{{\color{black}R_g^m}{{{I}}_m}{{{{\dot \omega }}}_m} + {\color{black}R_g^m}{{{\omega }}_m} \times ({{{I}}_m}{{{\omega }}_m})}
\end{bmatrix}}

+{\begin{bmatrix}
{{{m_m}{{[0,0,g]}^T}}}\\
\bf{0}
\end{bmatrix}}=\nonumber\\
&\qquad\qquad\quad
\underbrace{\begin{bmatrix}
{R_g^m}\begin{bmatrix}
0 & 0 & 0 & 0 \\
0 & 0 & 0 & 0 \\
1 & 1 & 1 & 1 \end{bmatrix}\\
{R_g^m}\begin{bmatrix}
0 & d & 0 & -d \\
-d & 0 & d & 0 \\
\frac{k_M}{k_F} & -\frac{k_M}{k_F} & \frac{k_M}{k_F} & -\frac{k_M}{k_F}
\end{bmatrix}
\end{bmatrix}}_{=:\widetilde{A}_m}
\underbrace{\begin{bmatrix}
F_1\\
F_2\\
F_3\\
F_4\end{bmatrix}}_{=:\widetilde{T}}
\label{eq:J1_Appx_B_quadrotor_dyn_redu}
\end{align}
where $\widetilde{A}_m$ and $\widetilde{T}$ represent the reduced matrix and vector of ${A}_m$ and $T$, respectively. Clearly, \eqref{eq:J1_Appx_B_quadrotor_dyn_redu} matches the results shown in~\cite{Mellinger2012} (i.e., the combination of (2.1), (2.2), (2.3), and (2.4) in~\cite{Mellinger2012}).

After $\widetilde{A}_m$ is computed using~\eqref{eq:J1_Appx_B_quadrotor_dyn_redu}, the whole body dynamics of the quadrotor with the 2-DOF robot arm (see~\autoref{fig:J1_appx_quadrotor}) can be easily derived from Appendix~\ref{appendix:J1_9DOFDynamics} (choosing the first two joints and links of the robot arm). Finally, the equations of motion can also be arranged in forms of \eqref{eq:J1_22} and \eqref{eq:J1_24}, where ${A}_m$ and $T$ are replaced by $\widetilde{A}_m$ and $\widetilde{T}$, respectively.
}

\section*{Acknowledgment}
The authors would like to knowledge the financial support of the Natural Sciences and Engineering Research Council of Canada (NSERC).

\bibliographystyle{IEEEtran}
\bibliography{IEEEabrv,Generalized_Flexible_Hybrid_Cable-Driven_Robot}

\vfill
\end{document}